\documentclass[11pt]{article}

\usepackage{pdfpages}
\usepackage{environ}
\usepackage{booktabs}
\usepackage{geometry}
\usepackage{natbib}
\usepackage{amsmath}
\usepackage{amssymb}
\usepackage{mathtools}
\usepackage{mathrsfs}
\usepackage{amsthm}

\usepackage{color}
\usepackage{algorithmic}
\usepackage{algorithm}
\usepackage{titletoc}
\usepackage{pdfpages}

\usepackage{hyperref,etoolbox}
\hypersetup{hidelinks}
\hypersetup{
colorlinks=true,
linkcolor=blue,
citecolor=blue
}

\DeclareMathOperator*{\argmax}{argmax}

\newcommand{\doi}{{\it do}}

\def\##1\#{\begin{align}#1\end{align}}
\def\$#1\${\begin{align*}#1\end{align*}}

\newcommand{\cA}{\mathcal{A}}
\newcommand{\cB}{\mathcal{B}}
\newcommand{\cC}{\mathcal{C}}

\newcommand{\cE}{\mathcal{E}}

\newcommand{\cS}{\mathcal{S}}

\newcommand{\cU}{\mathcal{U}}

\newcommand{\cP}{{\mathcal{P}}}
\newcommand{\cO}{\mathcal{O}}

\newcommand{\rT}{\mathscr{T}}

\newcommand{\DD}{\mathbb{D}}
\newcommand{\EE}{\mathbb{E}}

\newcommand{\II}{\mathbb{I}}

\newcommand{\PP}{\mathbb{P}}

\usepackage{thm-restate}

\usepackage{xcolor}

\usepackage[utf8]{inputenc} 
\usepackage[T1]{fontenc}    
\usepackage{hyperref}       
\usepackage{url}            
\usepackage{booktabs}       
\usepackage{amsfonts}       
\usepackage{nicefrac}       
\usepackage{microtype}      
\usepackage{xcolor}         
\usepackage{graphicx}
\usepackage{titlesec}
\usepackage{subfigure}

\usepackage{titletoc}

\usepackage{booktabs} 
\usepackage{caption}
\theoremstyle{plain}
\newtheorem{theorem}{Theorem}[section]

\newtheorem{lemma}[theorem]{Lemma}

\theoremstyle{definition}
\newtheorem{definition}[theorem]{Definition}
\newtheorem{assumption}[theorem]{Assumption}
\theoremstyle{remark}

\usepackage[textsize=tiny]{todonotes}
\usepackage{enumitem}
\usepackage{ifthen}
\newcommand{\compilehidecomments}{false}
\ifthenelse{ \equal{\compilehidecomments}{false} }{
	\newcommand{\longbo}[1]{}
	\newcommand{\yihan}[1]{}
	\newcommand{\nuoya}[1]{}
}{
	\newcommand{\longbo}[1]{{\color{purple}  [\text{Longbo:} #1]}}
	\newcommand{\yihan}[1]{{\color{brown} [\text{Yihan:} #1]}}
	\newcommand{\nuoya}[1]{{\color{red} [\text{Nuoya:} #1]}}
}

\newcommand{\compilefullversion}{true}
\ifthenelse{\equal{\compilefullversion}{false}}{%
	\newcommand{\OnlyInFull}[1]{}
	\newcommand{\OnlyInShort}[1]{#1}
}{%
	\newcommand{\OnlyInFull}[1]{#1}%
	\newcommand{\OnlyInShort}[1]{}%
}%

\title{Provably Safe Reinforcement Learning with Step-wise Violation Constraints}

\author{Nuoya Xiong\\
 IIIS, 
 Tsinghua University\\
 xiongny20@mails.tsinghua.edu.cn \\
 \and
            Yihan Du\\
IIIS, 
 Tsinghua University\\
 duyh18@mails.tsinghua.edu.cn \\
 \and
 Longbo Huang\\
 IIIS, 
 Tsinghua University\\
 longbohuang@tsinghua.edu.cn
}

\begin{document}

\maketitle

\begin{abstract}
We investigate a novel safe reinforcement learning problem with step-wise violation constraints. Our problem differs from existing works in that we focus on stricter step-wise violation constraints and do not assume the existence of safe actions, making our formulation more suitable for safety-critical applications that need to ensure safety in all decision steps but may not always possess safe actions, e.g., robot control and autonomous driving.
We propose an efficient algorithm SUCBVI, which guarantees $\widetilde{\mathcal{O}}(\sqrt{ST})$ or gap-dependent $\widetilde{\mathcal{O}}(S/\cC_{\mathrm{gap}} + S^2AH^2)$ step-wise violation and $\widetilde{\mathcal{O}}(\sqrt{H^3SAT})$ regret. Lower bounds are provided to validate the optimality in both violation and  regret performance with respect to the number of states $S$ and the total number of steps $T$. 
Moreover, we further study an innovative safe reward-free exploration problem with step-wise violation constraints. For this problem, we design algorithm SRF-UCRL to find a near-optimal safe policy, which achieves nearly state-of-the-art  sample complexity $\widetilde{\mathcal{O}}((\frac{S^2AH^2}{\varepsilon}+\frac{H^4SA}{\varepsilon^2})(\log(\frac{1}{\delta})+S))$, and guarantees $\widetilde{\mathcal{O}}(\sqrt{ST})$ violation during exploration.  Experimental results demonstrate the  superiority of our algorithms in safety performance and corroborate our theoretical results. 
\end{abstract}

\section{Introduction}

In recent years, reinforcement learning (RL)~\citep{sutton2018reinforcement} has become a powerful framework for  decision making and learning in unknown environments. 
Despite the ground-breaking success of RL in games~\citep{lanctot2019openspielgame}, 
recommendation systems~\citep{afsar2022reinforcementrecommand} 
 and complex tasks in simulation environments~\citep{zhao2020sim}, 
most existing RL algorithms focus on optimizing the cumulative reward and do not take into consideration the risk aspect, e.g., the agent runs into catastrophic situations during control. 
The lack of strong safety guarantees hinders the application of RL to broader safety-critical scenarios such as autonomous driving, robotics and healthcare.
For example, 
for robotic control in complex environments, it is crucial to prevent the robot from getting into  dangerous situations, e.g., hitting walls or falling into water pools, at all time. 

To handle the safety requirement, a common 
approach
is to formulate safety as a long-term expected violation constraint in each episode. 
This approach focuses on seeking a policy whose cumulative expected violation in each episode is below a certain threshold.
However, for applications where an agent needs to avoid disastrous situations throughout the decision process, e.g., 
a robot needs to avoid hitting obstacles at each step, merely reducing the long-term expected violation is not sufficient to guarantee safety.

Motivated by this fact, we investigate safe reinforcement learning with a more fine-grained constraint, called \emph{step-wise} violation constraint, which aggregates all nonnegative violations at each step 
(no offset between positive and negative violations permitted).
We name this problem Safe-RL-SW.  
Our step-wise violation constraint differs from prior expected violation constraint~\citep{wachi20CMDP, exploration2020CMDP, kalagar2021aaaiCMDP} 
 in two aspects: (i) Minimizing the step-wise violation enables the agent to learn an optimal policy that avoids unsafe regions deterministically, while reducing the expected violation only guarantees to find a policy with low expected violation, instead of a per-step zero-violation policy. (ii) Reducing the aggregated nonnegative violation allows us to have a risk control for each step, while a small cumulative expected violation can still result in a large cost at some individual step and cause danger, if other steps with smaller costs offset the huge cost. 

Our problem faces two unique challenges. First, the step-wise violation requires us to guarantee a small violation at each step, which demands very different algorithm design and analysis from that for the expected violation~\citep{wachi20CMDP,exploration2020CMDP, kalagar2021aaaiCMDP}. 
Second, in safety-critical scenarios, the agent needs to identify not only unsafe states but also potentially unsafe states, which are states that may appear to be safe but will ultimately lead to unsafe regions with a non-zero probability. 
For example, a self-driving car needs to learn to slow down or change directions early, foreseeing the potential danger in advance in order to ensure safe driving \citep{nips2021luoimaginethefuture}. 
Existing safe RL works focus mainly on the expected violation~\citep{wachi20CMDP, liu2021zeroviolation, wei2022triple}, or assume no potentially unsafe states by imposing the prior knowledge of a safe action for each state~\citep{anami2021safeLFA}. Hence, techniques in previous works cannot be applied to handle step-wise violations. 

    


To systematically handle these two challenges, we formulate safety as an unknown cost function for each state without assuming safe actions, and consider minimizing the step-wise violation instead of the expected violation. We propose a general algorithmic framework called \textbf{S}afe \textbf{UCBVI} 
  (SUCBVI). 
Specifically, in each episode, we first estimate the transition kernel and cost function in an optimistic manner, tending to regard a state as safe at the beginning of learning. After that, we introduce novel dynamic programming  to identify potentially unsafe states and determine safe actions, based on our  estimated  transition and costs. Finally, we employ the identified safe actions to conduct  value iteration. 
This mechanism can adaptively update  dangerous regions, and help the agent plan for the future, which keeps her away from all states that may lead to unsafe states.
As our estimation becomes more accurate over time, the safety violation becomes smaller and eventually converges to zero. 
Note that without the assumption of safe actions, the agent knows nothing about the environment at the beginning. Thus, it is impossible for her to achieve an absolute zero violation. Nevertheless, we show that SUCBVI can achieve sub-linear $\widetilde{\cO}(\sqrt{ST})$ cumulative violation or an 
$\widetilde{\cO}(S/\cC_{\mathrm{gap}} + S^2AH^2)$ gap-dependent  violation that is independent of $T$. This violation implies that as the RL game proceeds, the agent eventually learns how to avoid unsafe states. 
We also provide a matching lower bound to demonstrate the optimality of SUCBVI in both violation and regret.

Furthermore, we apply our step-wise safe RL framework to the  reward-free exploration \citep{jinchi20rewardfree} setting.  
In this novel safe reward-free exploration,
the agent needs to guarantee small step-wise violations during exploration, and also output a near-optimal safe policy.  Our algorithm achieves $\varepsilon$ cumulative step-wise violation during exploration, and also identifies a $\varepsilon$-optimal and safe policy. See the definition of $\varepsilon$-optimal and $\varepsilon$-safe policy in Section \ref{sec:formulation of safe-rfe}.
Another interesting application of our framework is safe zero-sum Markov games, which we discuss in Appendix \OnlyInFull{\ref{appendix:safe markov games}}\OnlyInShort{D}.

The main contributions of our paper are as follows.

\begin{itemize}[leftmargin=*]
    \item We formulate the safe RL with step-wise violation constraint problem (Safe-RL-SW), which models safety as a cost function for states and aims to minimize the cumulative step-wise violation.
    Our formulation is particularly useful for safety-critical applications where avoiding disastrous situations at each decision step is desirable, e.g., autonomous driving and robotics.
    \item We provide a general algorithmic framework SUCBVI, which is equipped with an innovative dynamic programming to identify  potentially unsafe states and distinguish safe actions. 
    We establish $\widetilde{\cO}(\sqrt{H^3SAT})$ regret and $\widetilde{\cO}(\sqrt{ST})$ or  $\widetilde{\cO}(S/\cC_{\mathrm{gap}} + S^2AH^2)$ gap-dependent violation guarantees, which exhibits the capability of SUCBVI in attaining high rewards while maintaining small violation. 

    \item We further establish  $\Omega(\sqrt{HST})$ regret and $\Omega(\sqrt{ST})$ violation  lower bounds for Safe-RL-SW.  The lower bounds demonstrate the optimality of algorithm SUCBVI  in  both regret minimization and safety guarantee, with respect to factors $S$ and $T$. 

    \item 
    We consider step-wise safety constraints in the reward-free exploration setting, which calls the Safe-RFE-SW problem. In this setting, we design an efficient algorithm SRF-UCRL,
    which ensures $\varepsilon$ step-wise violations  during exploration and plans a $\varepsilon$-optimal and $\varepsilon$-safe policy for any reward functions with probability at least $1-\delta$. 
    We obtain an $\widetilde{\cO}((\frac{S^2AH^2}{\varepsilon}+\frac{H^4SA}{\varepsilon^2})(\log(\frac{1}{\delta})+S))$ sample complexity and an  $\widetilde{\cO}(\sqrt{ST})$  violation guarantee for SRF-UCRL, which shows the efficiency of SRF-UCRL in sampling and danger avoidance even without reward signals. To the best of our knowledge, this work is the first to study the step-wise violation constraint in the RFE setting.
\end{itemize}

\section{Related Work}
\paragraph{Safe RL.} Safety is an important topic in RL, which has been extensively studied. 
The constrained Markov decision process (CMDP)-based approaches handle safety via cost functions, and aim to minimize the expected episode-wise violation, e.g., \citep{mingyu2019nipspolicyopt, wachi20CMDP, qiu2020upper,exploration2020CMDP, turchetta2020safe,ding2020natural,singh2020learning,kalagar2021aaaiCMDP,simao2021alwayssafe,ding2021provably}, or achieve zero episode-wise violation, e.g., \citep{liu2021zeroviolation, Bura2021nipszeroviolation, wei2022triple,  sootla2022enhancing}. 
 Apart from CMDP-based approaches, there are also other works that tackle safe RL by control-based approaches \citep{felix2017control,chow2018nipslycontrol,dalal2018safe,zhuoran2022enforcinghardconstraints}, policy optimization \citep{uchibe2007constrained,achiam2017constrained,tessler2018reward,liu2020ipo,stooke2020responsive} and safety shields \citep{alshiekh2018safe}. In recent years, there are also some works studying step-wise violation with  additional assumptions such as safe action \citep{anami2021safeLFA} or known state-action subgraph \citep{shi2023near}. More detailed comparisons are provided in \OnlyInFull{\ref{appendix:related}}\OnlyInShort{B}.
\vspace{-1em}
\paragraph{Reward-free Exploration with Safety.}
Motivated by sparse reward signals in realistic applications, reward-free exploration (RFE) has received extensive attention \citep{jinchi20rewardfree, ACT2020adaptiveRFE, Pierre2020Fastactive}. 
Recently,
\cite{miryoosefi2022simple,ruiquan2022safeRFE} also study RFE with safety constraints. However, \cite{miryoosefi2022simple} only considers the safety constraint after the exploration phase.   \cite{ruiquan2022safeRFE} considers the episode-wise constraint during the exploration phase, requiring an assumption of a known baseline policy. The detailed comparison is provided in Appendix \OnlyInFull{\ref{appendix:related}}\OnlyInShort{B}.

\section{The Safe MDP Model}
\vspace{-0.2em}
\textbf{Episodic MDP.} 
In this paper, we consider the finite-horizon episodic Markov Decision Process (MDP),  represented by a tuple $(\cS, \cA, H, \PP, r)$. Here $\cS$ is the state space, $\cA$ is the action space, and $H$ is the length of each episode. 
$\PP = \{\PP_h:\cS \times \cA \mapsto \triangle_{\cS}\}_{h \in [H]}$ is the transition kernel, and $\PP_h(s'|s,a)$ gives the transition probability from $(s,a)$ to $s'$ at step $h$. $r = \{r_h:\cS \times \cA \mapsto [0,1]\}_{h \in [H]}$ is the reward function, and $r_h(s,a)$ gives the reward of taking action $a$ in state $s$ at step $h$. 
A policy $\pi = \{\pi_h:\cS \mapsto \cA\}_{h \in [H]}$ consists of $H$ mappings from the state space to action space.
In each episode $k$,  the agent first chooses a policy $\pi^k$. At each step $h \in [H]$, the agent observes a state $s_h^k$, takes an action $a_h^k$, and then goes to a next state $s_{h+1}^k$ with probability $\PP_h(s_{h+1}^k\mid s_h^k,a_h^k)$. The algorithm executes $T = HK$ steps. Moreover,
 the state value function $V^{\pi}_h(s,a)$ and state-action value function $Q_h^\pi(s,a)$ for a policy $\pi$ can be defined as
\begin{small}
\begin{align*}
V^\pi_h(s) &:= \EE_{\pi} \Bigg[\sum_{h' = h}^H r_{h'}(s_{h'}, \pi_{h'}(s_{h'})) \Bigg|s_{h}=s \Bigg],\\
Q^\pi_h(s,a) &:= \EE_{\pi} \Bigg[\sum_{h' = h}^H r_{h'}(s_{h'}, \pi_{h'}(s_{h'})) \Bigg|s_{h}=s, a_{h}=a \Bigg].
\end{align*}
\end{small}

\textbf{Safety Constraint.} 
To model unsafe regions in the environment, similar to \cite{wachi20CMDP, yu2022reachability},  we
define a safety cost function $c:\cS \mapsto [0,1]$.
Let $\tau \in [0,1]$ denote the safety threshold. A state is called {\em safe} if $c(s)\le \tau$, and called {\em unsafe} if $c(s)>\tau$. Similar to \cite{exploration2020CMDP, anami2021safeLFA}, when the agent arrives in a state $s$, she will receive a cost signal $z(s) = c(s) + \zeta$, where $\zeta$ is an independent, zero-mean and 1-sub-Gaussian noise. 
Denote $(x)_{+} = \max\{x,0\}$.
The violation in state $s$ is defined as $(c(s)-\tau)_{+}$, and the cumulative step-wise violation till episode $K$ is 
\begin{equation}\label{eq:violation definition}
    C(K) = \sum_{k=1}^K\sum_{h=1}^H  (c(s_h^k)-\tau)_{+} .
\end{equation}
Eq.~\eqref{eq:violation definition} represents the accumulated step-wise violation during training. When the agent arrives in state $s_h^k$ at step $h$ in episode $k$, she will suffer violation $(c(s_h^k)-\tau)_{+}$.  
This violation setting is significantly different from the previous CMDP setting \citep{ qiu2020upper,ding2020natural,exploration2020CMDP,ding2021provably,wachi20CMDP,liu2021optpess,kalagar2021aaaiCMDP}. They study the episode-wise expected violation   $C'(K) = \sum_{k=1}^K (\EE[\sum_{h=1}^H c(s_h^k, a_h^k)]-\mu)$.
There are two main differences between the step-wise violation and episode-wise expected violation: (i) First, the expected violation constraint allows the agent to get into unsafe states occasionally. Instead, the step-wise violation constraint enforces the agent to stay in safe regions at all time. (ii) Second, in the episode-wise constraint, the violation $c(s_h,a_h)-\tau$ at each step is allowed to be positive or negative, and they can cancel out in one episode. Instead, we consider a \emph{nonnegative} function $(c(s)-\tau)_+$ for each step in our step-wise violation, which imposes a stricter constraint.



Define $\cU:=\{s \in \cS\mid c(s)>\tau\}$ 
as the set of all \emph{unsafe states}. Let $\iota = \{s_1,a_1,\cdots,s_H,a_H\}$ denote a trajectory. Since a feasible policy needs to satisfy the constraint at each step,  we define the set of feasible policies as
    $\Pi = \{\pi\mid \Pr\{\ \exists \ h \in [H], s_h \in \cU\mid \iota \sim \pi\}=0\}.$
The feasible policy set $\Pi$ consists of all policies under which one never reaches any unsafe state in an episode.

\textbf{Learning Objective.} 
In this paper, we consider the regret minimization objective.
Specifically, define $\pi^* = \argmax_{\pi \in \Pi} V_1^{\pi}$, $V^* = V^{\pi^*}$ and $Q^* = Q^{\pi^*}$. 
The regret till $K$ episodes is then defined as 
    \begin{equation*}
         R(K) = \sum_{t= 1}^K  (V^*_1(s_1) - V^{\pi^k}_1(s_1) ) ,
    \end{equation*}
    where $\pi^k$ is the policy taken in episode $k$.
    Our objective is to minimize $R(K)$ to achieve a good performance, and minimize the violation $C(K)$ to guarantee the safety at the same time. 

\section{Safe RL with Step-wise Violation Constraints}\label{sec:model and analysis of safe RL}
\vspace{-0.2em}
\subsection{Assumptions and Problem Features}\label{Section:Assumptions and Problem Features}
Before introducing our algorithms, we first state the  important assumptions and problem features for Safe-RL-SW. 

Suppose $\pi$ is a feasible policy. Then, if we arrive at $s_{H-1}$ at step $H-1$,
$\pi$ needs to select an action that guarantees $s_H\notin \cU$.
Define the \textit{transition set} 
$\Delta_h(s,a) = \{s'\mid \PP_h(s' \!\mid\! s,a)>0\}$ for any $\forall (s,a,h) \!\in\! \cS \!\times\! \cA\times [H]$,
which represents the set of possible next states after taking action $a$ in state $s$ at step $h$. 
Then, at the former step $H-1$, the states $s$ is potentially unsafe if it satisfies that  $\Delta_{H-1}(s,a)\cap \cU\neq \emptyset$ for all $a \in \cA$. (i.e., no matter taking what action, there is a positive probability of transitioning to an unsafe next state).
Therefore, we can recursively define the set of  \emph{potentially unsafe states at step $h$} as 
\begin{equation}
\begin{footnotesize}
\label{eq:defCh}
    \cU_h = \cU_{h+1}\cup \{s\mid\ \forall\  a \in \cA, \Delta_{h}(s,a)\cap \cU_{h+1}\neq \emptyset\},
\end{footnotesize}
\end{equation}
where $\cU_H=\cU.$ Intuitively, if we are in a state $s_h\in \cU_h$ at step $h$, no action  can be taken to completely avoid reaching potentially unsafe states $s_{h+1}\in \cU_{h+1}$ at step $h+1$. Thus, in order to completely prevent from getting into unsafe states $\cU$ throughout all steps, one needs to avoid potentially unsafe states in $\cU_h$ at step $h$.
From the above argument, we have that  $s_1\notin \cU_1$ is equivalent to the existence of feasible policies. The detailed proof is provided in Appendix \OnlyInFull{\ref{appendix:detailed analysis of assum}}\OnlyInShort{F}. 
Thus, we make the following necessary  assumption.

\begin{assumption}[Existence of feasible policies] \label{assum:existence of feasible pol}
The initial state $s_1$ satisfies
$s_1 \notin \cU_1$.
\end{assumption}
\vspace{-0.5em}
For any $s \in \cS$ and $h \in [H-1]$, we define the set of safe actions for state $s$ at step $h$ as
\begin{align}\label{eq:def of Asafe(s)}
    A_h^{safe}(s) = \{a \in \cA\mid \Delta_h(s,a)\cap \cU_{h+1}=\emptyset\},
\end{align}
and let $A_H^{safe}(s) = \cA$.
$A_h^{safe}(s)$ stands for the set of all actions at step $h$ which will not lead to potentially unsafe states in $\cU_{h+1}$.
 Here, $\{\cU_h\}_{h \in [H]}$ 
 and $\{A_h^{safe}(s)\}_{h \in  [H]}$ are defined by dynamic programming: If we know all possible next state sets $\{\Delta_h(s,a)\}_{h\in[H]}$ and unsafe state set $\cU=\cU_H$, we can calculate all potentially unsafe state sets $\{\cU_h\}_{h \in [H]}$ and safe action sets $\{A_h^{safe}(s)\}_{h\in[H]}$, and choose feasible policies to completely avoid  unsafe states.

\subsection{Algorithm SUCBVI}\label{sec: Alg for safe RL}

Now we present our main algorithm \textbf{S}afe \textbf{UCBVI} (SUCBVI), which is based on previous classic RL algorithm UCBVI \citep{UCBVI2017}, and equipped with a novel dynamic programming to identify potentially unsafe states and safe actions.  
The pseudo-code is shown in Algorithm~\ref{alg:safe RL}. 
First, we provide some intuitions about how SUCBVI works. At each episode, we first estimate all the unsafe states based on the historical data. Then, we perform a dynamic programming procedure introduced in Section~\ref{Section:Assumptions and Problem Features} and calculate the safe action set $A_h^{safe}(s)$ for each state $s$. Then, we perform value iteration in the estimated safe action set. As the estimation becomes more accurate, SUCBVI will eventually avoid potentially unsafe states
 and achieve both sublinear regrets and violations.
 
Now we begin to introduce our algorithm. In the beginning, we initialize $\Delta_h(s,a) = \emptyset$ for all $(h, s, a) \in [H]\times \cS\times \cA$. It implies that the agent considers all actions to be safe at first, because no action will lead to unsafe states from the agent's perspective. 
In each episode, we first estimate the empirical cost $\hat{c}(s)$ based on historical cost feedback $z(s)$, and regard state $s$ as safe if $\bar{c}(s) = \hat{c}(s)-\beta >\tau$ for some bonus term $\beta$, which aims to guarantee $\bar{c}(s)\le c(s)$ (Line~\ref{line:update emp cost}). 
Then, we calculate the estimated unsafe state set $\cU_H^k$, which is a subset of the true unsafe state set $\cU=\cU_H$ with a high probability by optimism. With $\cU_{H}^k$ and $\Delta_h^k(s,a)$, we can estimate potentially unsafe state sets $\cU_h^k$ for all $h \in [H]$ by Eq.~\eqref{eq:defCh} recursively.

Then, we perform value iteration to compute the optimistically estimated optimal policy. 
Specifically, for any hypothesized safe state $s \notin \cU_h^k$, we update the greedy policy on the estimated safe actions, i.e., $\pi_h^k(s) = \max_{a \in A_h^{k,safe}(s)}Q_h^k(s,a)$ (Line \ref{line:update policy}).
On the other hand, for any  hypothesized unsafe state $s_h \in \cU_h$, since there is no action that can completely avoid unsafe states, we ignore safety costs and simply update the policy by $\pi_h^k(s) = \max_{a \in \cA}Q(s,a)$ (Line~\ref{line:update policy 2}).
 After that, we calculate the estimated optimal policy $\pi^k$ for episode $k$, and the agent follows $\pi^k$ and collects a trajectory. Then, we update $\Delta_h^{k+1}(s,a)$ by incorporating the observed state $s_{h+1}^k$ into the set $\Delta_h^k(s_h^k,a_h^k)$. Under this updating rule, it holds that $\Delta_h^k(s,a)\subseteq \Delta_h(s,a)$ for all $(s,a) \in \cS\times \cA$.
\begin{algorithm}[t]
\begin{algorithmic}[1]
\caption{SUCBVI}
	\label{alg:safe RL}
 \STATE Initialize: $\Delta_h^1(s,a) = \emptyset$, $N_h^{1}(s,a)=N_h^1(s,a,s') = 0$ for all $s \in \cS$, $a \in \cA$, $s' \in \cS$ and $h \in [H]$. 
 
 \FOR{$k=1,2,\cdots,K$} 
     \STATE \textcolor{blue}{Update the optimistic estimates of cost.}
    \STATE Update the empirical cost $\hat{c}(s)$ and calculate $\bar{c}(s) = \hat{c}(s) - \beta(N^k(s), \delta)$ for all $s \in \cS$. $\triangleright$ \label{line:update emp cost}
    \STATE Define $\cU_{H}^k = \{s\mid \bar{c}(s) > \tau\}$ \label{line:decide unsafe state} and $\cU_h^k = \cU_{h+1}^k\cup \{s\mid \forall a \in A, \Delta_h^k(s,a)\cap \cU_{h+1}^k\neq\emptyset\}$ for all $h\in[H]$ recursively. Calculate $\{A_h^{k,safe}(s)\}_{h \in [H]}$ by Eq.~\eqref{eq:def of Asafe(s)} with $\{\cU_h^k\}_{h \in [H]}$.
     \\$\triangleright$ \textcolor{blue}{Perform value iteration with previous estimates.}
     \FOR{$h=H,H-1,\cdots,1$}
        \FOR{$s \in \cS$}
            \FOR{$a \in \cA$}
            \STATE Compute $\hat{\PP}^k_h(s'\mid s,a) = \frac{N_h^k(s,a,s')}{N_h^k(s,a)}$. 
            \STATE $Q_h^k(s,a) =\min\{H, r_h(s,a) + \sum_{s'}\hat{\PP}_h^k(s'\mid s,a)V_{h+1}^k(s')+ \alpha(N_h^k(s,a))\}.$
            \ENDFOR
                
            \IF{$s \notin \cU_h^{k}$}
            \STATE $V_h^k(s) = \max_{a \in A_h^{k,safe}(s)} Q_h^k(s,a), \pi_h^k(s)$ = $\arg\max_{a \in A_h^{k,safe}(s)}Q_h^k(s,a)$ \label{line:update policy}.
            \ELSE 
            \STATE  $V_h^k(s) = \max_{a \in \cA} Q_h^k(s,a),\pi_h^k(s)$ = $\arg\max_{a \in \cA}Q_h^k(s,a)$.\label{line:update policy 2}
            \ENDIF
        \ENDFOR
     \ENDFOR
     \FOR{$h = 1,2,\cdots,H$}
        \STATE Take action $a_h^k = \pi_h^k(s_h^k)$ and observe state $s_{h+1}^k.$
        \\$\triangleright$ \textcolor{blue}{Update the estimates of $\Delta(s,a)$.}
        \STATE $\Delta_h^{k+1}(s_h^k,a_h^k) = \Delta_h^k(s_h^k,a_h^k) \cup \{s_{h+1}^k\}$\label{eq:update delta}. Increase $N_h^{k+1}(s_h^k,a_h^k), N_h^{k+1}(s_h^k,a_h^k,s_{h+1}^k)$ by $1$.
     \ENDFOR
     \ENDFOR
\end{algorithmic}
\end{algorithm}

%

The performance of Algorithm~\ref{alg:safe RL} is summarized below in  Theorem~\ref{thm:safe RL}.

\begin{theorem}\label{thm:safe RL}
Let $\alpha(n,\delta) = 7H\sqrt{\frac{\ln(5SAHK/\delta)}{n}}$ and $\beta(n,\delta) = \sqrt{\frac{2}{n}\log (SK/\delta)}$. 
With probability at least $1-\delta$, the regret and step-wise violation of Algorithm~\ref{alg:safe RL} are bounded by 
\begin{align*}
    R(K) = \widetilde{\cO}(\sqrt{H^2SAT}),\qquad 
    C(K) = \widetilde{\cO}(\sqrt{ST}+S^2AH^2).
\end{align*}
Moreover, if $\cC_{\mathrm{gap}} \triangleq\min_{s \in \cU} (c(s)-\tau)_+ >0$, we have $C(K) = \widetilde{\cO}(S/\cC_{\mathrm{gap}}+S^2AH^2)$.
\end{theorem}
Theorem~\ref{thm:safe RL} shows that SUCBVI achieves both sublinear regret and violation. Moreover, when all the unsafe states have a large cost compared to the safety threshold $\tau$, i.e., $\cC_{\mathrm{gap}} = \min_{s \in \cU}(c(s)-\tau)_+>0$ is a constant, we can distinguish the unsafe states easily and get a constant violation. In particular, when $\tau = 1$, Safe-RL-SW degenerates to the unconstrained MDP and Theorem \ref{thm:safe RL} maintains the same regret $\widetilde{\cO}(\sqrt{H^2SAT})$ as UCBVI-CH \citep{UCBVI2017}, while CMDP algorithms \citep{exploration2020CMDP, liu2021optpess} suffer a larger $\widetilde{\cO}(\sqrt{H^3S^3AT})$ regret. 

We provide the analysis idea of Theorem~\ref{thm:safe RL} here. 
First, by the updating rule of $\Delta_h^k(s,a)$, we show that $\cU_h^k$ and $A_h^{k,safe}(s)$ have the following crucial properties:
$
     \cU_h^k\subseteq \cU_h,\ \   A_h^{safe}(s) \subseteq A_h^{k,safe}(s).
$
 Based on this property, we can prove that $Q^*_h(s,a)\le Q_h^k(s,a)$ and $V^*_h(s)\le V_h^k(s)$ for all $(s,a)$, and then apply the regret decomposition techniques to derive the regret bound. 

 Recall that $\pi^k$ is a feasible policy with respect to the estimated unsafe state set $\cU_H^k$ and transition set $\{\Delta_h^k(s,a)\}_{h \in [H]}$.  If the agent takes policy $\pi^k$ and the transition follows $\Delta_h^k(s,a)$, i.e., $s_{h+1}^k\in \Delta_h^k(s,a)$, the agent never arrive any estimated unsafe state $s \in \cU_H^k$ in this episode.  Hence the agent will suffer at most $\widetilde{\cO}(\sqrt{ST})$  step-wise violation or $\widetilde{\cO}(S/\cC_{\mathrm{gap}} + S^2AH^2)$ gap-dependent bounded violation.
 Yet,  the situation $s_{h+1}\in \Delta_h^k(s,a)$ does not always hold for all $h \in [H]$. 
 In the case when $s_{h+1}^k\notin \Delta_h^k(s,a)$, we add the newly observed state $s_{h+1}^k$ into $\Delta_h^{k+1}(s_h^k,a_h^k)$ (Line~\ref{eq:update delta}).
 We can show that this case  appears at most $S^2AH$ times, and thus incurs  $O(S^2AH^2)$ additional violation. 
 Combining the above two cases,  
 the total violation can be   upper bounded.

 %
 

\section{Lower Bounds for Safe-RL-SW} \label{sec:lower-bound}
In this section, we provide a matching lower bound for Safe-RL-SW in Section~\ref{sec:model and analysis of safe RL}. 
The lower bound shows that if an algorithm always achieves a sublinear regret in Safe-RL-SW, it must incur an $\Omega(\sqrt{ST})$ violation. This result matches our upper bound in Theorem~\ref{thm:safe RL}, showing that SUCBVI achieves the optimal violation performance. 
\begin{theorem}\label{thm:lower bound of violation}
    If an algorithm has an expected regret $\mathbb{E}_{\pi}[R(K)]\le \frac{HK}{24}$ for all MDP instances, there exists an MDP instance in which the algorithm suffers expected violation $\mathbb{E}_{\pi}[C(K)] = \Omega(\sqrt{ST})$.
\end{theorem}
 
Now we validate the optimality in terms of regret.
Note that if we do not consider safety constraints, the  lower bound for classic RL~\citep{Obsnad2016lowerboundRL} 
can be applied to our setting. Thus,  we also have a $\Omega(\sqrt{T})$ regret lower bound. To understand the essential hardness brought by safety constraints, we further investigate whether safety constraints will lead to $\Omega(\sqrt{T})$ regret, given that we can achieve a $o(\sqrt{T})$ 
regret on some good instances without the safety constraints.
\begin{theorem}\label{thm:lower bound regret}
      For any $\alpha\in (0,1)$,  there exists a parameter $n$ and $n$ MDPs $M_1,\dots,M_n$ satisfying that: 

    1. If we do not consider any constraint, there is an algorithm which achieves $\widetilde{\cO}(T^{(1-\alpha)/2})$ regret compared to the unconstrained optimal policy on all MDPs.

    2. If we consider the safety constraint, any algorithm with $O(T^{1-\alpha})$ expected violation will achieve $\Omega(\sqrt{HST})$ regret compared to the  constrained optimal  policy on one of MDPs.
\end{theorem}
Intuitively, Theorem~\ref{thm:lower bound regret} shows that if one achieves sublinear violation, she  must suffer at least $\Omega(\sqrt{T})$ regret even if she can achieve $o(\sqrt{T})$ regret without considering constraints. 
This theorem demonstrates the hardness particularly brought by the step-wise constraint, and corroborates the optimality of our results. Combining with Theorem \ref{thm:lower bound of violation}, the two lower bounds show an essential trade-off between the violation and performance.
\section{Safe Reward-Free Exploration with Step-wise Violation Constraints}
\subsection{Formulation of Safe-RFE-SW}\label{sec:formulation of safe-rfe}
In this section, we consider Safe RL in the reward-free exploration (RFE) setting~\citep{jinchi20rewardfree,ACT2020adaptiveRFE,Pierre2020Fastactive} called Safe-RFE-SW, to show the generality of our proposed framework. 
In the RFE setting, the agent does not have access to reward signals and only receives random safety cost feedback $z(s) = c(s) + \zeta$.
To impose safety requirements, Safe-RFE-SW requests the agent to keep small safety violations during exploration, and outputs a near-optimal safe policy after receiving the reward function.

\begin{definition}[$(\varepsilon,\delta)$-optimal safe algorithm for Safe-RFE-SW]
    An algorithm is $(\varepsilon,\delta)$-optimal safe for Safe-RFE-SW if it outputs the triple $(\hat{\PP}, \hat{\Delta}(s,a), \hat{\cU}_H)$ 
    such that for any reward function $r$, with probability at least $1-\delta$,
    \begin{equation}
    \begin{small}\label{eq:definition of safe RFE}
        V_1^*(s_1;r) - V_1^{\hat{\pi}^*} (s_1;r) \le \varepsilon,\qquad 
        \EE_\pi \left[\sum_{h=1}^H (c(s_h)-\tau)_{+}\right]\le \varepsilon ,
    \end{small}
    \end{equation}
    where $\hat{\pi}^*$ is the optimal feasible policy with respect to $(\hat{\PP}, \hat{\Delta}(s,a), \hat{\cU}_H, r)$, and $V(s;r)$ is the value function under reward function $r$. We say that a policy is \textit{$\varepsilon$-optimal} if it satisfies the left inequality in Eq.~\eqref{eq:definition of safe RFE} and \textit{$\varepsilon$-safe} if it satisfies the right inequality in Eq.~\eqref{eq:definition of safe RFE}. We measure the performance by the number of episodes used before the algorithm terminates, i.e., {\em sample complexity}. Moreover, the cumulative step-wise violation till episode $K$ is defined as $C(K)=\sum_{k=1}^K \sum_{h=1}^H (c(s_h^k)-\tau)_{+}$.
    In Safe-RFE-SW, our goal is to design an $(\varepsilon,\delta)$-optimal safe algorithm, and minimize both  sample complexity and violation.
\end{definition}

\subsection{Algorithm SRF-UCRL}
The Safe-RFE-SW problem requires us to consider extra safety constraints for both the exploration phase and final output policy, which needs new algorithm design and techniques compared to previous RFE algorithms.
Also, the techniques for  Safe-RL-SW in Section \ref{sec: Alg for safe RL} are not sufficient for guaranteeing the safety of output policy, because SUCBVI only guarantees a step-wise violation during exploration. 


We design an efficient algorithm \textbf{S}afe \textbf{RF-UCRL} (SRF-UCRL), which builds upon previous RFE algorithm RF-UCRL~\citep{ACT2020adaptiveRFE}. SRF-UCRL distinguishes potentially unsafe states and safe actions by backward iteration, and establishes a new uncertainty function to  guarantee the safety of output policy. 
Algorithm~\ref{alg:safe UCRL} illustrates the procedure of SRF-UCRL.
Specifically, in each episode $k$, we first execute a policy $\pi_h^k$ computed from previous episodes, and then update the estimated next state set $\{\Delta^k_h(s,a)\}_{h \in [H]}$ and unsafe state set $\cU_{H}^k$ by optimistic estimation. Then, we use Eq.~\eqref{eq:defCh} to calculate the unsafe state set $\cU_h$ for all steps $h \in [H]$. 
After that, we update the uncertainty function $\overline{W}^k$ defined in Eq.~\eqref{eq:defW} below and compute the policy $\pi^{k+1}$ that maximizes the uncertainty to encourage more exploration in the next episode.

Now we provide the definition of the \textit{uncertainty function}, which measures the estimation error between the empirical MDP and true MDP.
For any safe state-action pair $s \notin \cU_{h}, a \in A_h^{k, safe}(s)$, 
we define 
\begin{small}
\begin{align}
\overline{W}_h^k(s,a)= \min\left\{H,M(N_h^k(s,a), \delta) +  \sum_{s'}\hat{\PP}_h^k(s'\mid s,a) \max_{b \in A^{k,safe}_{h+1}(s')}\overline{W}_{h+1}^k(s',b)\right\}.\label{eq:defW}
\end{align}
\end{small}
where 
$M(N_h^k(s,a), \delta) = 2H\sqrt{\frac{2\gamma(N_h^k(s,a),\delta)}{N_h^k(s,a)}}+\frac{SH\gamma(N_h^k(s,a),\delta)}{N_h^k(s,a)},$ and $\gamma(n,\delta)$ is a logarithmic term that is formally defined in Theorem \ref{thm:safe UCRL}.
For other state-action pairs $(s,a)$, we  relax the restriction of $b \in \cA_{h+1}^{k,safe}(s')$ to $b \in \cA$ in the $\max$ function.
Our algorithm stops when $\overline{W}^K_1(s_1,\pi^{K}(s_1))$ 
shrinks to within $\varepsilon/2$. 
Compared to previous RFE works (\cite{ACT2020adaptiveRFE, Pierre2020Fastactive}), our uncertainty function has two distinctions. First, for any safe state-action pair $(s,a) \in (\cS\setminus\cU_h, A_h^{k,safe}(s))$, Eq.~\eqref{eq:defW} considers only safe actions $A_{h+1}^{k,safe}(s)$, which guarantees that the agent focuses on safe policies. 
Second, Eq.~\eqref{eq:defW} incorporates another term $(SH\gamma(N_h^k(s,a),\delta)/N_h^k(s,a))$ to control the expected violation for feasible policies.
Now we present our result for Safe-RFE-SW. 
\begin{theorem}\label{thm:safe UCRL}
    Let $\gamma(n,\delta) = 2(\log (2SAH/\delta)+(S-1)\log(e(1+n/(S-1)))$, Algorithm~\ref{alg:safe UCRL} is a $(\varepsilon,\delta)$-PAC algorithm with sample complexity at most\footnote{Here $\widetilde{\cO}(\cdot)$ ignores all $\log S, \log A, \log H, \log(1/\varepsilon)$ and $\log(\log(1/\delta))$ terms.}
  \begin{equation*}
  \begin{small}
      K = \widetilde{\cO}\left(\left(\frac{S^2AH^2}{\varepsilon}+\frac{H^4SA}{\varepsilon^2}\right)\left(\log\left(\frac{1}{\delta}\right)+S\right)\right),
  \end{small}
  \end{equation*}
    The step-wise violation of Algorithm~\ref{alg:safe UCRL} during exploration is  $C(K) = \widetilde{\cO}(S^2AH^2+\sqrt{ST}) .$
\end{theorem}
Compared to previous work~\citep{ACT2020adaptiveRFE} with  $\widetilde{\cO}((H^4SA/\varepsilon^2)(\log(1/\delta)+S))$  sample complexity, our result has an additional term $\widetilde{\cO}((S^2AH^2/\varepsilon)(\log(1/\delta)+S))$. This term is due to the additional safety requirement for the final output policy, which was not considered in previous RFE algorithms ~\citep{ACT2020adaptiveRFE, Pierre2020Fastactive}. 
When $\varepsilon$ and $\delta$ are sufficiently small, the leading term is $\widetilde{\cO}((H^4SA/\varepsilon^2)\log(1/\delta))$, which implies that our algorithm satisfies the safety constraint without suffering additional regret.\footnote{\cite{Pierre2020Fastactive} improve the result by a factor $H$ and replace $\log(1/\delta)+S$ by $\log(1/\delta)$ via the Bernstein-type inequality.
While they do not consider the safety constraints,
we believe that similar improvement can also be applied to our framework, without significant changes in our analysis. However, since our paper mainly focuses on tackling the safety constraint for reward-free exploration, we use the Hoeffding-type inequality to keep the succinctness of our statements.} 

\begin{algorithm}[t]
\begin{algorithmic}[1]
	\caption{SRF-UCRL}
	\label{alg:safe UCRL}

	\STATE Initialize: $k = 1$, $\overline{W}_h^0(s,a)=H$, $\pi_h^1(s) = a_1$ for all $s \in \cS$, $a \in \cA$ and $h \in [H]$.
    \WHILE{$\overline{W}_h^{k-1}(s_1,\pi_1^{k}(s_1))\le \varepsilon/2$}
    \FOR{$h=1,\cdots,H$\label{line:beginfor}}
    \STATE Observe state $s_h^k$. Take action $a_h^k=\pi_h^k (s_h^k)$ and observe $s_{h+1}^k.$ \\
    $\triangleright$ \textcolor{blue}{Update the optimistic estimate of cost.}
    \STATE Calculate empirical cost $\hat{c}(s)$ and $\bar{c}(s) = \hat{c}(s) - \beta(N^k(s), \delta)$ for all $s\in \cS$.
    \\$\triangleright$ \textcolor{blue}{Update the estimates of $\Delta_h(s,a)$ and $\cU_h$.}
    \STATE Update $\Delta_h^{k}(s_h^k,a_h^k) = \Delta_h^{k-1}(s_h^k,a_h^k) \cup \{s_{h+1}^k\}$ and  $\{\cU_h^k\}_{h \in [H]}$ by Eq.~\eqref{eq:defCh}. 
    \STATE Calculate $A_h^{k,safe}(s)$ for all $s \in \cS$ by Eq.~\eqref{eq:def of Asafe(s)}.
    \ENDFOR\label{line:endfor}
    \STATE Update $N_h^k(s,a,s'), N_h^k(s,a)$ and 
    $\hat{\PP}_h^k$ for all $s \in \cS, a \in \cA, s' \in \cS$ and $h \in [H]$.
    \STATE Compute $\overline{W}^{k}$ according to Eq.~\eqref{eq:defW}.  \\$\triangleright$ \textcolor{blue}{Calculate the greedy policy.} 
    \FOR{$h \in [H], s \in \cS$}
    \IF{$A_h^{k,safe}(s)\neq \emptyset$}
    \STATE $\pi_h^{k+1}(s) = \arg\max_{a \in A_h^{k, safe}(s)} \overline{W}_h^{k} (s,a).$
    
    \ELSE \STATE $\pi_h^{k+1}(s) =\arg\max_{a \in A} \overline{W}_h^{k} (s,a).$
        
    \ENDIF
    \ENDFOR
    
    \STATE Set $k = k+1$.
    \ENDWHILE
    \STATE {\bfseries return} the tuple $(\hat{\PP}_h^{k-1}, \Delta_h^{k-1}(s,a), \cU_H^{k-1})$.
    \end{algorithmic}
\end{algorithm}
\subsection{Analysis for Algorithm SRF-UCRL}
 For the analysis of step-wise violation (Eq. \eqref{eq:violation definition}), 
 similar to  algorithm SUCBVI, algorithm SRF-UCRL estimates the next state set $\Delta_h(s,a)$ and potentially unsafe state set $\cU_h$, which guarantees a $\widetilde{\cO}(\sqrt{ST})$ step-wise violation.
 Now we give a proof sketch for the $\varepsilon$-safe property of output policy (Eq. \eqref{eq:definition of safe RFE}).
First,  if $\pi$ is a feasible policy for $(\hat{\PP}^k, \Delta^k(s,a), \{\cU_h^k\}_{h \in [H]})$ and
$s_{h+1} \in \Delta_h^K(s_h,a_h)$ for all $h \in [H]$, 
the agent who follows policy $\pi$ will only visit the estimated safe states.  Since each estimated safe state only suffers $O(1/\sqrt{N_t(s_h, \pi_h(s_h))})$ violation, the violation led by this situation is bounded by $\overline{W}_1^k(s,\pi_1(s_1))$.
Next, we bound the probability that $s_{h+1}\notin \Delta^K_h(s_h,a_h)$ for some step $h \in [H]$. For any state-action pair $(s,a)$, if there is a probability $\PP(s'\mid s,a)\ge \widetilde{\cO}(\log(S/\delta)/N_h^K(s,a))$ that the agent transitions to next state $s'$ from $(s,a)$ at step $h$, the state $s'$ is put into $\Delta_h^K(s,a)$ with probability at least $1-\delta/S$. Then, 
we can expect that 
all such states are put into our estimated next state set $\Delta_h^K(s,a)$. Thus, the probability that $s_{h+1}\notin \Delta_h^K(s_h,a_h)$ is no larger than $\widetilde{\cO}(S\log(S/\delta)/N_h^K(s,a))$ by a union bound over all possible next states $s'$. Based on this argument, $\overline{W}_h^k(s_h,a_h)$ is an upper bound for the total probability that $s_{h+1}\notin \Delta^K_h(s_h,a_h)$ for some step $h$. This will lead to additional $\overline{W}_1^K(s_1,\pi_1(s_1))$ expected violation.
Hence the expected violation of 
 output policy is upper bounded by  $2\overline{W}_1^{K}(s_1,\pi_1^{K}(s_1)) \le \varepsilon$.  The complete proof is provided in Appendix \OnlyInFull{\ref{appendix:proof}}\OnlyInShort{C}.


\section{Experiments}

\vspace{-0.2em}
In this section, we provide experiments for Safe-RL-SW and Safe-RFE-SW to validate our theoretical results. For Safe-RL-SW, we compare our algorithm SUCBVI with a classical RL algorithm UCBVI~\citep{UCBVI2017} and three state-of-the-art CMDP algorithms OptCMDP-bonus~\citep{exploration2020CMDP} , Optpess~\citep{liu2021optpess} and Triple-Q \citep{wei2022triple}. 
For Safe-RFE-SW, 
 we report the average reward in each episode and cumulative step-wise violation.
For Safe-RFE-SW, we compare our algorithm SRF-UCRL with a state-of-the-art RFE algorithm RF-UCRL~\citep{ACT2020adaptiveRFE}. We do not plot the regret because unconstrained MDP or CMDP algorithms do not guarantee step-wise violation. Applying them to the step-wise constrained setting can lead to negative or large
regret and large violations, making the results meaningless. Detailed experiment setup is in Appendix \OnlyInFull{\ref{appendix:experiment}}\OnlyInShort{G}.

\begin{figure*}[t]
	\centering
	\begin{minipage}[b]{0.245\linewidth}
		\centering
		\includegraphics[width=1\linewidth]{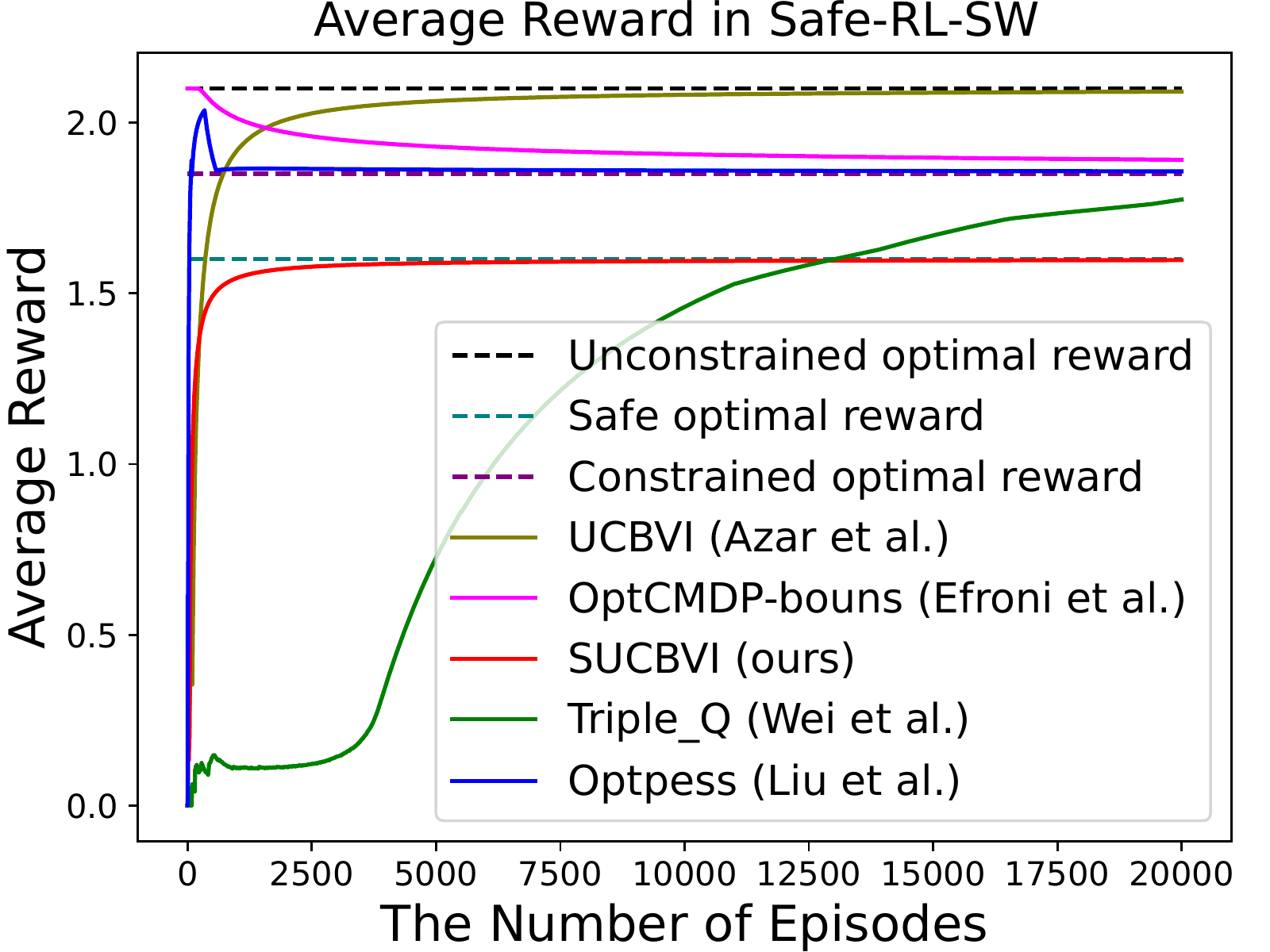}
		\label{Rewards of UCBVI}
	\end{minipage}
	\begin{minipage}[b]{0.245\linewidth}
		\centering
		\includegraphics[width=1\linewidth]{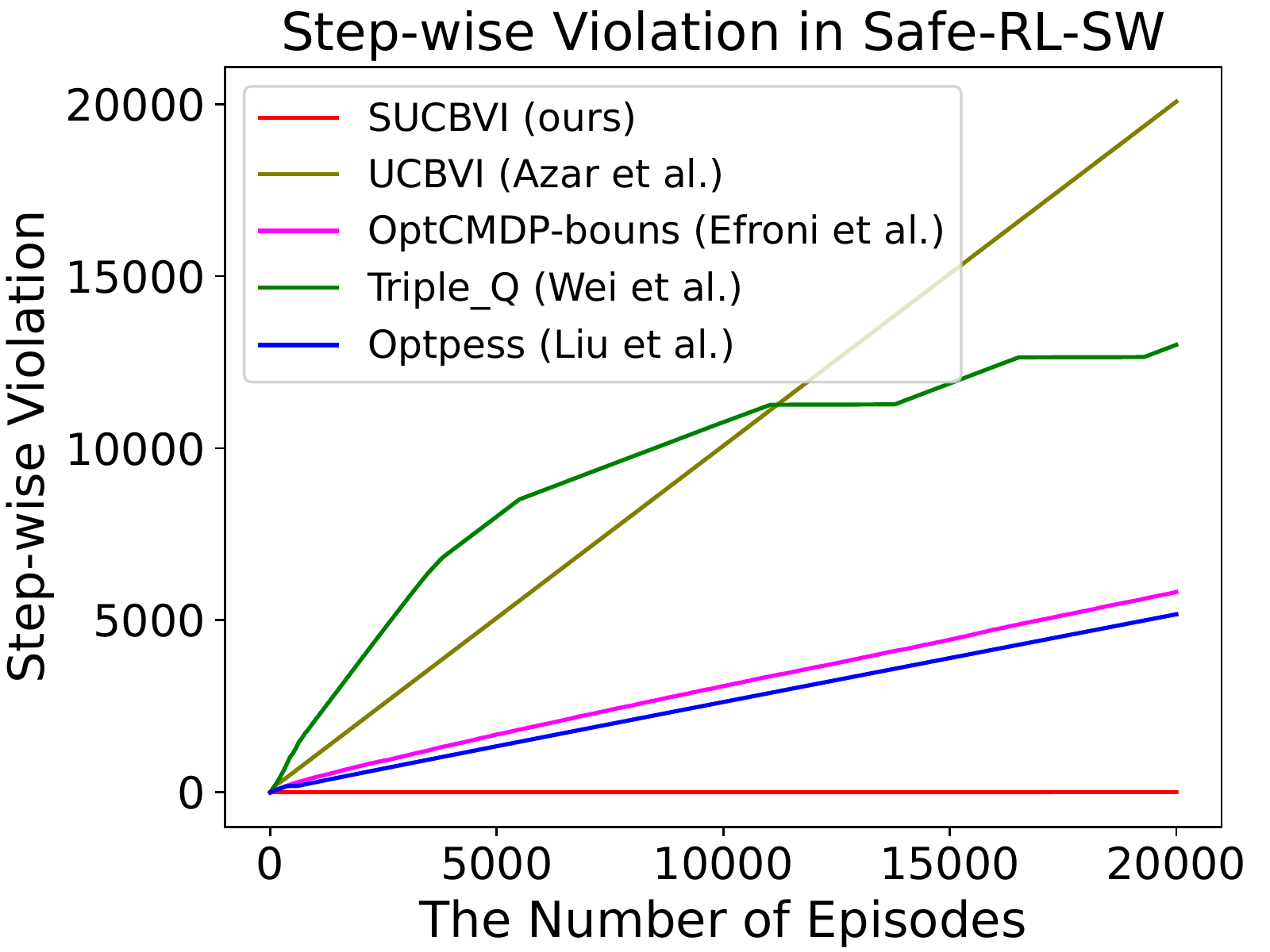}
		\label{violations of UCBVI}
	\end{minipage}
	\begin{minipage}[b]{0.245\linewidth}
		\centering
		\includegraphics[width=1\linewidth]{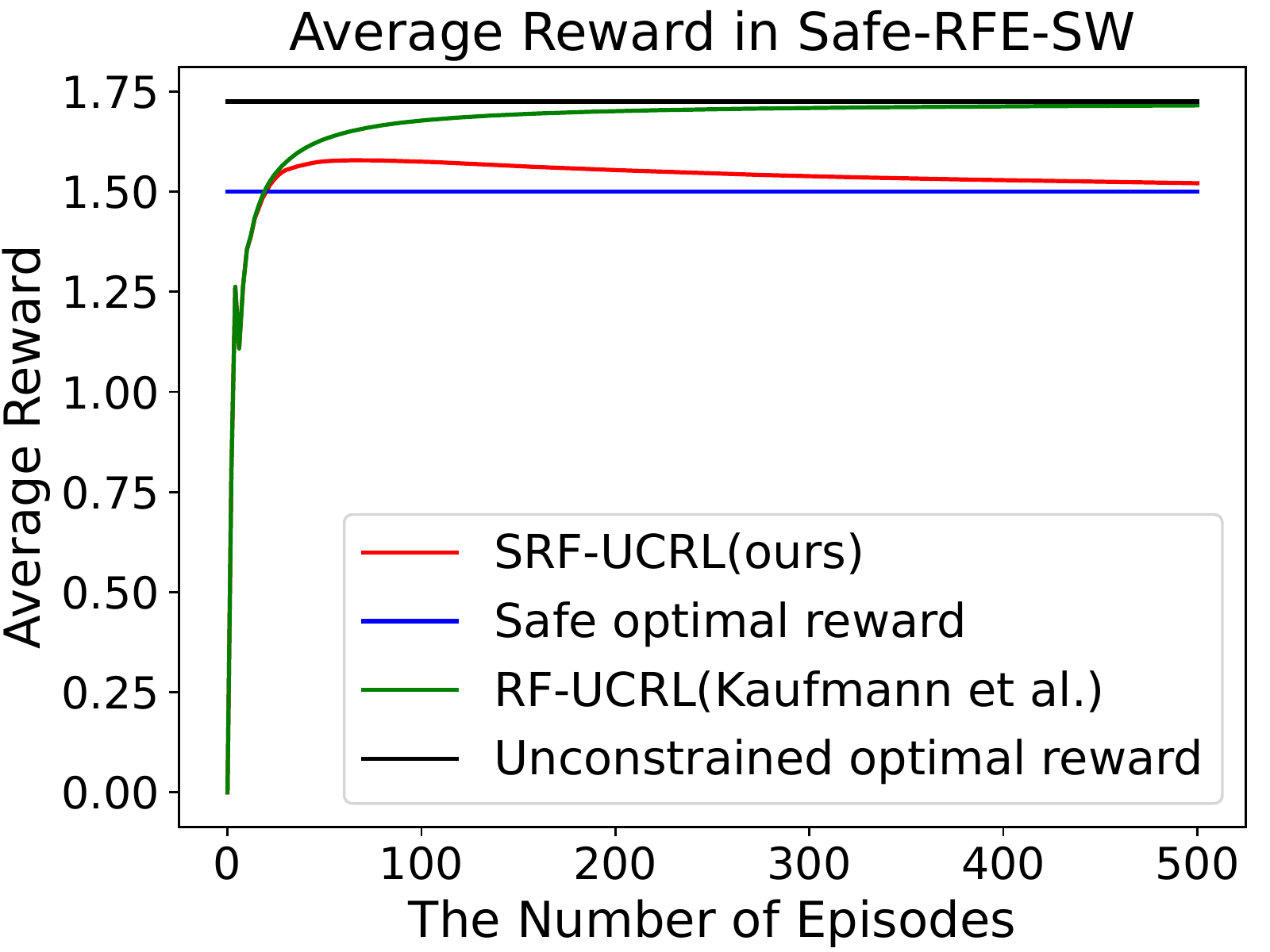}
		\label{rewards of SUCRL}
	\end{minipage}
	\begin{minipage}[b]{0.245\linewidth}
		\centering
		\includegraphics[width=1\linewidth]{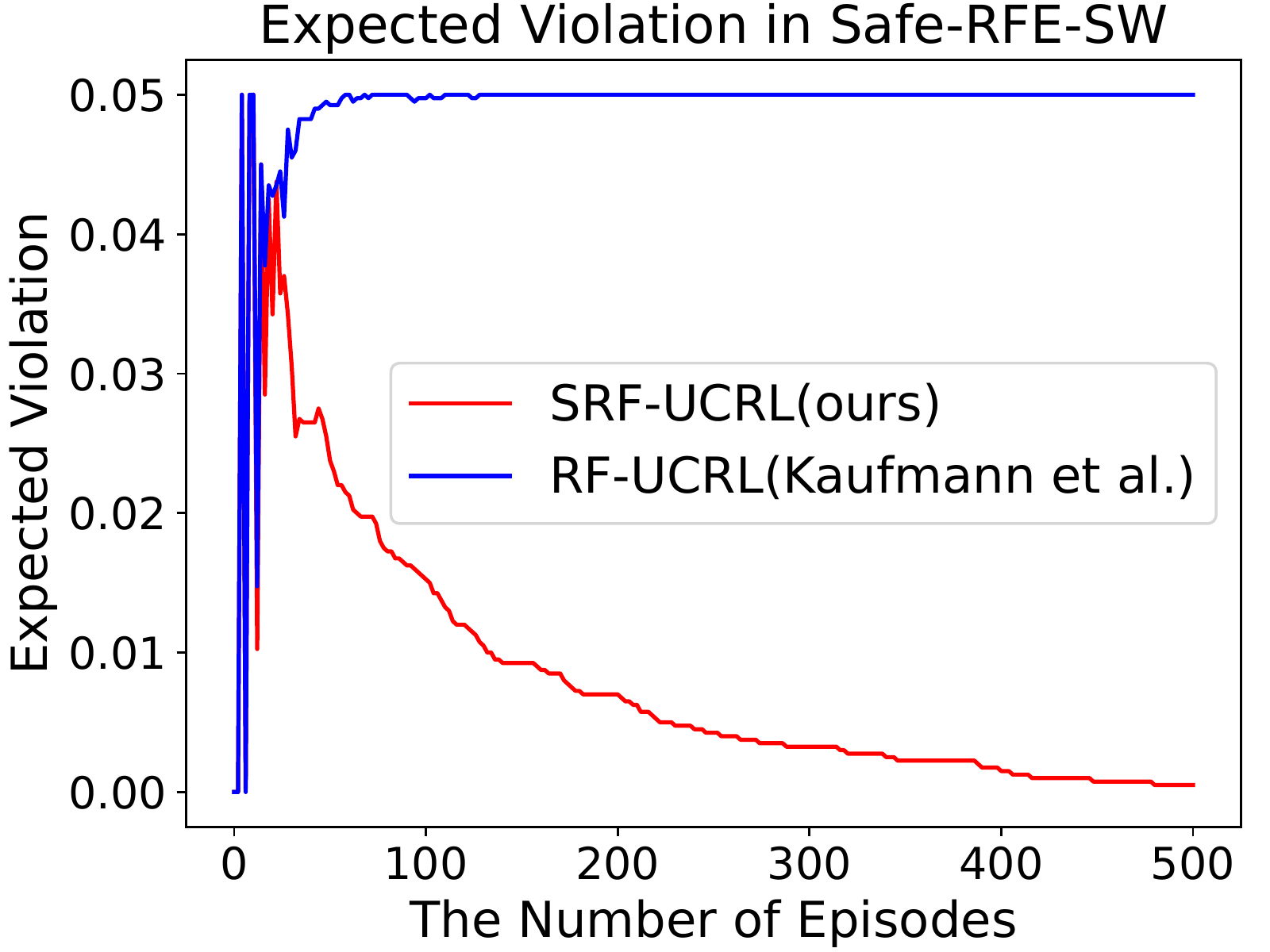}
		\label{Violations of SUCRL}
	\end{minipage}
    
    \caption{
    Experimental results for Safe-RL-SW and Safe-RFE-SW.  
    The left two figures show the average rewards and step-wise violations of algorithms SUCBVI,  UCBVI~\citep{UCBVI2017}, OptCMDP-bonus~\citep{exploration2020CMDP}, Triple-Q~\citep{wei2022triple} and Optpess~\citep{liu2021optpess}. The right two figures show the reward and expected violation of the policies outputted by algorithms SRF-UCRL and RF-UCRL~\citep{ACT2020adaptiveRFE}.  
    }
    \label{fig:whole}
\end{figure*}

As shown in Figure~\ref{fig:whole}, the rewards of SUCBVI and SRF-UCRL converge to the optimal rewards under safety constraints (denoted by "safe optimal reward"). In contrast, the rewards of UCBVI and UCRL converge to the optimal rewards without safety constraints (denoted by  "unconstrained optimal reward"), and those of CMDP algorithms converge to a policy with low expected violation (denoted by "constrained optimal reward"). 
For Safe-RFE-SW, the expected violation of output policy of SRF-UCRL converges to zero while that of RF-UCRL does not. 
This corroborates the ability of SRF-UCRL in finding policies that simultaneously achieve safety and high rewards.

\section{Conclusion}
In this paper, we investigate a novel safe reinforcement learning problem with step-wise safety constraints. We first provide an algorithmic framework SUCBVI to achieve both $\widetilde{\cO}(\sqrt{H^3SAT})$ regret and $\widetilde{\cO}(\sqrt{ST})$ step-wise or $\widetilde{\cO}(S/\cC_{\mathrm{gap}} + S^2AH^2)$ gap-dependent bounded violation that is independent of $T$. Then, we provide two lower bounds to validate the optimality of SUCBVI in both violation and regret in terms of $S$ and $T$.
Further, we extend our framework to the safe RFE with a step-wise violation and provide an algorithm SRF-UCRL 
that identifies a near-optimal safe policy given any reward function $r$ and guarantees $\widetilde{\cO}(\sqrt{ST})$ violation during exploration. 


\OnlyInFull{
\onecolumn
\appendix
\newpage

\ \\
\centerline{\begin{Large} \textbf{Appendix}\end{Large}}

\begin{large} \startcontents
\printcontents{}{1}{}\end{large}

\vspace{.2in}


\section{Related Works}\label{appendix:related}
\paragraph{Safe RL except for CMDP-approaches}
Apart from the CMDP approaches,  (\cite{mingyu2019nipspolicyopt, qiu2020upper,exploration2020CMDP, ding2021provably,kalagar2021aaaiCMDP,simao2021alwayssafe,liu2021zeroviolation, Bura2021nipszeroviolation,wei2022triple}), there are several other ways to achieve safety in reinforcement learning.
The control-based approaches (\cite{felix2017control,chow2018nipslycontrol, liu2020ipo,zhuoran2022enforcinghardconstraints}) adopt barrier functions to keep the agent away from unsafe regions and impose hard constraints. The policy-optimization style approaches are also commonly studied to solve constrained RL problems. (\cite{uchibe2007constrained,achiam2017constrained,tessler2018reward,liu2020ipo,stooke2020responsive})

There are some other papers investigating the safety of RL problems.
\cite{alshiekh2018safe} represents the safe state by the reactive system, and uses shielding to calculate and restrict the agent within a safe trajectory completely. The main difference between their work and our work is that we need to dynamically update the estimated safe state, while they require to know the mechanism and state to calculate the shield.
\cite{dalal2018safe} considers restricting the safe action by projecting the action into the closest safe actions. They achieve this goal by solving a convex optimization problem on the continuous action set. However, in their paper, they do not consider the situation where a state can have no safe actions. To be more specific, they do not consider the situation when the convex optimization problem has no solutions. 
\cite{le2019batch} considers the decision-making problem with a pre-collected dataset. 
 \cite{turchetta2020safe} and \cite{sootla2022enhancing} both consider cumulative cost constraints rather than step-wise constraints. The former uses a teacher for intervention to keep the agent away from the unsafe region, while the latter encourages safe exploration by augmenting a safety state to measure safety during training. 

\paragraph{RFE with Safety} 
Motivated by sparse reward signals in realistic applications, reward-free exploration (RFE) has been proposed  in \cite{jinchi20rewardfree}, and further developed in~\cite{ACT2020adaptiveRFE, Pierre2020Fastactive}. In RFE, the agent explores the environment without reward signals. After enough exploration, the reward function is given, and the agent needs to plan a near-optimal policy based on his knowledge collected in exploration. Safety is also important in RFE: We need to not only guarantee that our outputted policy is safe, but also ensure small violations during exploration. 

In recent years,
\cite{ruiquan2022safeRFE} studies RFE with safety constraints. Our work differs from this work in the following aspects: (i) They allow different reward functions during exploration and after exploration, and the agent can directly know the true costs
during training. 
In our work, the cost function is the same during and after exploration, but the agent can only observe the noisy costs of a state when she arrives at that state. (ii) They require prior knowledge of safe baseline policies, while we do not need such an assumption. 
(iii) They consider the zero-expected violation during exploration, while we focus on keeping small step-wise violations.

\section{Proofs of Main Results}
\label{appendix:proof}
\subsection{Proof of Theorem~\ref{thm:safe RL}}
\label{sec:proof of safe RL}
We start the proof with the confidence bound of optimistic estimation:
\begin{lemma}\label{lemma:e1confidence}
    With probability at least $1-\delta$, the empirical optimistic estimation of cost function $\overline{c}^k(s) = \hat{c}^k(s) - \sqrt{\frac{2}{N^k(s)}\log(SK/\delta)}$ are always less than the true cost $c(s)$ for any episode $k \in [K]$. In other words, let the event be 
    \begin{align}\label{event:e1}
        \cE_{1} = \left\{\forall \ k \in [K] \ \mbox{and}\  s \in \cS, \overline{c}^k(s)\le c(s)\right\},
    \end{align}
    then $\Pr(\cE_1^c)\le \delta.$
\end{lemma}

\begin{proof}
    Since we observe $c(s) + \varepsilon$ with 1-subgaussian noise $\varepsilon$, by Hoeffding's inequality~\cite{hoeffding1994probability}, for a fixed episode $k$ and state $s$, the statement 
    \begin{align*}
        |\hat{c}^k(s) - c(s)|\le \sqrt{\frac{2}{N^k(s)}\log\left(\frac{SK}{\delta}\right)}
    \end{align*}
    holds with probability at least $1-\frac{\delta}{SK}$. Taking a union bound on all $s \in \cS$ and $k\in [K]$, we have 
    $$c(s) \ge \hat{c}^k(s) - \sqrt{\frac{2}{N^k(s)}\log\left(\frac{SK}{\delta}\right)} = \overline{c}^k(s),$$ and
    the Lemma~\ref{lemma:e1confidence} has been proved.
\end{proof}

\begin{lemma}
    Let the event $\cE_2$ and $\cE_3$ are 
    \begin{small}
    \begin{align*}
        \cE_2 &= \left\{\forall \ s, a, s', k,h, |\hat{\PP}_h(s'\mid s,a) -\PP_h(s'\mid s,a)|\right.\\&\left. \hspace{5em} \le \sqrt{\frac{2\PP_h(s'\mid s,a)}{N_h^k(s,a)}\log\left(\frac{2S^2AHK}{\delta}\right)}+\frac{2}{3N_h^k(s,a)}\log\left(\frac{2S^2AHK}{\delta}\right)\right\}\\
        \cE_3 &= \left\{\sum_{k=1}^K \sum_{h=1}^H \left(P_h(V_{h+1}^k - V_{h+1}^{\pi^k})(s_h^k,a_h^k) - (V_{h+1}^k - V_{h+1}^{\pi^k})(s_{h+1}^k) \right)
        \le \sqrt{\frac{H^3K}{2}\log(1/\delta)}\right\}
    \end{align*}
    \end{small}
    Then $\Pr(\cE_2^c)\le \delta$, $\Pr(\cE_3^c)\le \delta$.
\end{lemma}

\begin{proof}
    This lemma is very standard in analysis of UCBVI algorithm.
    The $\Pr(\cE_2^c)$ can be derived by Bernstein's inequality, and $\Pr(\cE_3^c)\le \delta$ can be derived by Azuma-hoeffding inequality.  
\end{proof}

The following lemma shows the optimistic property we discussed in the paper.

\begin{lemma}
    Under the event $\cE_1$, $\cU_h^k\subseteq \cU_h$ and $A_h^{safe}(s)\subseteq A_h^{k,safe}(s) $ for all $s \in \cU_h^k$ and $h \in [H]$.
\end{lemma}
\begin{proof}
    Under the event $\cE_1$, for all $s \in \cU_{H}^k$, $c(s)\ge \overline{c}(s) >\tau$ . So $s \in \cU_{H}^*$. Now from our definition of $\Delta_h^K(s,a)$, we can easily have $\Delta^k_h(s,a) \subseteq \Delta_h(s,a)$ for all step $h \in [H]$ because we add the possible next state one by one into estimated set $\Delta_h^K(s,a)$. 
    
    Now we prove the lemma by induction. When $h = H$, $\cU_{H}^k\subseteq \cU_{H}^*$. Assume that $\cU_{h+1}^k \subseteq \cU_{h+1}^*$, then 
    \begin{align*}
        A_h^{safe}(s) &= \{a \in \cA\mid \Delta_h(s,a)\cap \cU_{h+1}^* = \emptyset\}\\
        &\subseteq \{a \in \cA\mid \Delta_h^K(s,a)\cap \cU_{h+1}^* = \emptyset\}\\
        &\subseteq \{a \in \cA\mid \Delta_h^K(s,a)\cap \cU_{h+1}^k = \emptyset\}\\
        & = A_h^{k,safe}(s),
    \end{align*}
    and 
    \begin{align*}
        \cU_h^k &= \cU_{h+1}^k \cup \{s\mid \exists\  a \in \cA, \Delta_h^K(s,a)\cap \cU_{h+1}^k=\emptyset\}\\
        & \subseteq \cU_{h+1}^*\cup\{s\mid \exists\  a \in \cA, \Delta_h(s,a)\cap \cU_{h+1}^*=\emptyset\}\\
        &\subseteq \cU_h.
    \end{align*}
Hence we complete the proof by induction.
\end{proof}
Similar to UCBVI algorithm, we prove that $Q_h^k(s,a)$ and $V_h^k(s)$ are always greater than $Q_h^*(s)$ and $V_h^*(s)$.

\begin{lemma}
    With probability at least $1-\delta$, choosing $b(n,\delta) = 7H\sqrt{\frac{\ln(5SAT/\delta)}{n}}$, we will have 
    \begin{align}
        Q_h^k(s,a)\ge Q_h^*(s,a), \ \ V_h^k(s)\ge V_h^*(s),\nonumber
    \end{align}
    for any state $s \in \cS$, action $a \in \cA$, episode $k$ and the step $h \in [H]$.
\end{lemma}
\begin{proof}
    We proof this lemma by induction. First, consider $h =H+1$, then $Q_{H+1}^*(s,a) = Q_{H+1}^k(s,a)= 0$ and $V_{H+1}^*(s) = V_{H+1}^k(s) = 0$ for all $k \in [K]$. Then our statements hold for $h = H+1$.
    Now fixed an episode $k$ and assume this statement holds for $h+1$. If $Q_h^k(s,a) = H$, our proof is completed. Then assume $Q_h^k(s,a)\le H$ and
    \begin{align*}
        Q_h^k(s,a)-Q_h^*(s,a) &= (\hat{\PP}_h^k)(V_{h+1}-V^*)(s,a) + (\hat{\PP}_h^k-\PP_h)(V^*)(s,a) + b(N_h^k(s,a))\\
        &\le (\hat{\PP}_h^k-\PP_h)(V^*)(s,a) + b(N_h^k(s,a))\\\
        &\le 0
    \end{align*}
    where the first inequality is induction hypothesis, and the second inequality follows naturally by Chernoff-Hoeffding's inequality.
    Then for $s \in \cU_h^k$
    \begin{align*}
        V_h^k(s) = \max_{a \in \cA}Q_h^k(s,a)\ge \max_{a \in \cA}Q^*(s,a) \ge V_h^*(s).
    \end{align*}
    For $s\notin \cU_h^k,$ since we know $\cU_h\subseteq \cU_h^k$, $s\notin \cU_h$. Then 
    \begin{align}
        V_h^k(s) = \max_{a \in A_h^{k,safe}(s)}Q_h^k(s,a) &\ge \max_{a \in A_h^{k,safe}(s)}Q_h^*(s,a)\nonumber\\
        &\ge \max_{a \in A_h^{safe}(s)}Q_h^*(s,a)\label{eq:maxexplain}\\
        & = V_h^*(s).\nonumber
    \end{align}
    where Eq.~\eqref{eq:maxexplain} is because $A_h^{safe}(s)\subseteq A_h^{k,safe}(s)$ by our optimistic exploration mechanism. 
\end{proof}

\paragraph{Regret} Now we can start to prove our main theorem. 
\begin{proof}
Having the previous analysis, our proof is similar to UCBVI algorithm.  Note that the regret can be bounded by UCB:
\begin{align*}
    R(K) = \sum_{k=1}^K (V_1^*(s_1) - V_1^{\pi^k}(s_1)) \le \sum_{k=1}^K (V_1^k(s_1) - V_1^{\pi^k}(s_1)).
\end{align*}
Because the action taking by $\pi^k$ is equal to action taking at episode $k$, we will have 
\begin{align*}
    V_h^k(s_h^k) - V_h^{\pi^k}(s_h^k) = (Q_h^k - Q_h^{\pi^k})(s_h^k, a_h^k), 
\end{align*}
where $a_h^k = \pi_h^k(s_h^k)$. Then 
\begin{align}
    V_h^k(s_h^k) - V_h^{\pi^k}(s_h^k) &= (Q_h^k - Q_h^{\pi^k})(s_h^k, a_h^k)\nonumber\\
    &= (\hat{\PP}_h^k(V_{h+1}^k)-\PP_h(V_{h+1}^{\pi^k}))(s_h^k,a_h^k) + b(N_h^k(s_h^k,a_h^k),\delta)\nonumber\\
    & = (\hat{\PP}_h^k -\PP_h)(V_{h+1}^*)(s_h^k,a_h^k) +(\hat{\PP}_h^k -\PP_h)(V_{h+1}^k- V_{h+1}^{*})(s_h^k,a_h^k) \nonumber\\&\ \ \ \ \ \ \ +P_h(V_{h+1}^k-V_{h+1}^{\pi^k}))(s_h^k,a_h^k) + b(N_h^k(s_h^k,a_h^k),\delta).\label{eq:decompose}
\end{align}
First, by Chernoff-Hoeffding's bound, $(\hat{\PP}_h^k -\PP_h)(V_{h+1}^*)(s_h^k,a_h^k) \le b(N_h^k(s_h^k,a_h^k),\delta).$ Second, by the standard analysis of UCBVI (Theorem 1 in UCBVI), under the event $\cE_{2}$, we will have 
\begin{align*}
    &\ \ \ \ \ (\hat{\PP}_h^k -\PP_h)(V_{h+1}^k- V_{h+1}^{*})(s_h^k,a_h^k) \\&\le \sum_{s' \in \cS}\left( \sqrt{\frac{2\PP_h(s'\mid s_h^k,a_h^k)\iota}{N_h^k(s_h^k,a_h^k))}} +\frac{2\iota}{3N_h^k(s_h^k,a_h^k)} \right)(V_{h+1}^k- V_{h+1}^{*})(s')\\
    &\le \sum_{s'\in \cS}\left( \frac{\PP_h(s'\mid s_h^k,a_h^k)}{H} +\frac{H\iota}{2N_h^k(s_h^k,a_h^k)}+\frac{2\iota}{3N_h^k(s_h^k,a_h^k)} \right)(V_{h+1}^k- V_{h+1}^{*})(s')\\
    &\le \frac{1}{H}\PP_h(V_{h+1}^k - V_{h+1}^{*})(s_h^k,a_h^k) + \frac{2H^2S\iota}{N_h^k(s_h^k,a_h^k)}\\
    &\le \frac{1}{H}\PP_h(V_{h+1}^k - V_{h+1}^{\pi^k})(s_h^k,a_h^k) + \frac{2H^2S\iota}{N_h^k(s_h^k,a_h^k)},
\end{align*}
where the first inequality holds under the event $\cE_2$, and the second inequality is derived by $\sqrt{ab}\le \frac{a+b}{2}$, and the last inequality is because $V_{h+1}^{\pi^k}\le V_{h+1}^*$.
Thus combining with Eq.~\eqref{eq:decompose}, we can get
\begin{align*}
    V_h^k(s_h^k) - V_h^{\pi^k}(s_h^k)\le \left(1+\frac{1}{H}\right)\PP_h(V_{h+1}^k - V_{h+1}^{\pi^k})(s_h^k,a_h^k) + 2b(N_h^k(s_h^k,a_h^k),\delta) + \frac{2H^2S\iota}{N_h^k(s_h^k,a_h^k)}.
 \end{align*}
 Denote 
 \begin{align*}
     \alpha_h^k &=\PP_h(V_{h+1}^k - V_{h+1}^{\pi^k})(s_h^k,a_h^k) - (V_{h+1}^k - V_{h+1}^{\pi^k})(s_h^K,a_h^k).\\
     \beta_h^k &= \frac{H^2S\iota}{N_h^k(s_h^k,a_h^k)}.\\
     \gamma_h^k &= b(N_h^k(s_h^k,a_h^k),\delta).
 \end{align*}
 we have 
 \begin{align*}
      V_h^k(s_h^k) - V_h^{\pi^k}(s_h^k)&\le \left(1+\frac{1}{H}\right)(V_{h+1}^k - V_{h+1}^{\pi^k})(s_{h+1}^k) + \left(1+\frac{1}{H}\right)\alpha_h^k + 2\beta_h^k + 2\gamma_h^k\\&\le\left(1+\frac{1}{H}\right)(V_{h+1}^k - V_{h+1}^{\pi^k})(s_{h+1}^k) + 2\alpha_h^k + 2\beta_h^k + 2\gamma_h^k,
 \end{align*}
 then by recursion, 
 \begin{align*}
     V_1^k(s_1) - V_1^{\pi^k}(s_1)&\le \sum_{h=1}^H\left(1+\frac{1}{H}\right)^H(2\alpha_h^k + 2\beta_h^k+2\gamma_h^k)\\
     &\le e\cdot \sum_{h=1}^H (2\alpha_h^k + 2\beta_h^k+2\gamma_h^k).
 \end{align*}
 Under the event $\cE_3$ with Bernstein's inequality, 
 $\sum_{k=1}^K \sum_{h=1}^H \alpha_h^k \le \sqrt{\frac{H^3K}{2}\log(1/\delta)}$, and 
 \begin{align}
     \sum_{k=1}^K \sum_{h=1}^H \beta_h^k  = \sum_{k=1}^K \sum_{h=1}^H \frac{H^2S\iota}{N_h^k(s_h^k,a_h^k)} \le H^3S^2A\iota \log (HK). 
 \end{align}
 Also,  for $\gamma_h^k$, we can get
 \begin{align*}
     \sum_{k=1}^K \sum_{h=1}^H \gamma_h^k &= \sum_{k=1}^K \sum_{h=1}^H 7H\sqrt{\frac{\iota}{N_h^k(s_h^k,a_h^k)}}\\
     &\le 7H\sqrt{\iota}\sum_{(s,a) \in \cS\times \cA}\sum_{i=1}^{N^k(s,a)}\sqrt{1/i}\\
     &\le 14H\sqrt{\iota}\sum_{(s,a) \in  \cS\times \cA}\sqrt{N^k(s,a)}\\
     &\le 14H\sqrt{\iota SAT}.
 \end{align*}
 Then the regret will be bounded by $O(\sqrt{H^2SAT})$. 
 \paragraph{Violation}
 Now we turn to consider the violation during the process.

 First, if the agent follows the $\pi^k$ at episode $k$, and $s_{h+1}^k \in \Delta_h^K(s_h^k,a_h^k)$ for all step $h \in [H-1]$, from the definition of $\{\cU_h\}_{1\le h\le H+1}$ and $\Delta_h^K(s_h^k,a_h^k)$, all states $s_1^k, s_2^k,\cdots,s_{H}^k \notin \cU_{H}^k$ are safe. Since $\cU_{H}^k = \{s\mid \overline{c}(s)>\tau\}$, 
 for each state $s \in \{s_1^k, s_2^k,\cdots,s_{H}^k\}$, we can have $$c(s)-\tau\le \overline{c}^k(s) + 2\sqrt{\frac{2}{N^k(s)}\log(SK/\delta)}-\tau\le 2\sqrt{\frac{2}{N^k(s)}\log(SK/\delta)},$$ where the first inequality from the event $\cE_1$ and the second inequality is because $s\notin \cU_{H}^k$. Then 
 \begin{align*}
     \sum_{h=1}^{H} (c(s_h^k)-\tau)_{+} \le \sum_{h=1}^{H} 2\sqrt{\frac{2}{N^k(s_h^k)}\log(SK/\delta)}.
 \end{align*}
 Then we consider the situation that $s_{h+1}^k \notin \Delta_h^K(s_h^k,a_h^k)$ for some step $h \in [H]$ at episode $k \in [K]$. In this particular case, we can only bound the violation on this episode by $H$ because the future state are not in control. Fortunately, in this situation we will add at least one new state $s_{h+1}^k$ to $\Delta_h^K(s_h^k,a_h^k)$, then $|\Delta_h^{k+1}(s_h^k,a_h^k)| \ge |\Delta_h^{k}(s_h^k,a_h^k)|+1$, and the summation $$\sum_{(h,s,a) \in [H]\times \cS\times\cA}|\Delta_h^{k+1}(s,a)|\ge \sum_{(h,s,a) \in [H]\times \cS\times\cA}|\Delta_h^{k}(s,a)|+1.$$ 
 Note that for any episode $k$, the summation has a trivial upper bound
 $$\sum_{(h,s,a) \in [H]\times \cS\times\cA}|\Delta_h^{k}(s,a)| \le S^2AH,$$

 this situation will also appears for at most $S^2AH$ episode and leads at most $S^2AH^2$ extra violation.

 Hence the total violation can be upper bounded by 
 \begin{align}
     \sum_{k=1}^K\sum_{h=1}^{H} (c(s_h^k)-\tau)_{+}&\le S^2AH^2+\sum_{k=1}^K \sum_{h=1}^{H} \left(2\sqrt{\frac{2}{N^k(s_h^k)}\log(SK/\delta)}\wedge 1\right)\nonumber\\
     &\le S^2AH^2+ 2\sqrt{2\log(SK/\delta)}\sum_{k=1}^K \sum_{h=1}^{H} \left(\sqrt{\frac{1}{N^k(s_h^k)\vee 1}}\right)\nonumber\\
     &\le S^2AH^2+ 2\sqrt{2\log(SK/\delta)}\sum_{s \in \cS}\sum_{i=1}^{K}\left(\frac{N^{i+1}(s)-N^i(s)}{\sqrt{N^i(s)\vee 1}} \right)\label{eq:sum}\\
     &\le S^2AH^2+ 2\sqrt{8\log(SK/\delta)}\sum_{s \in \cS} \sqrt{N^K(s)}\nonumber\\
     &\le S^2AH^2+ 2\sqrt{8\log(SK/\delta)}\sqrt{ST}\nonumber\\
     & = \widetilde{\cO}(S^2AH^2+\sqrt{ST})\nonumber
 \end{align}
 with probability at least $1-3\delta$. The inequality Eq.~\eqref{eq:sum} is derived by Lemma 19 in \cite{Jaksch2010RL}. By replacing $\delta$ to $\delta/3$ we can complete our proof.

 Moreover, for the unsafe state set $\cU$, recall that $d = \min_{s \in \cU}(c(s)-\tau)>0$. For the state $s_h^k$,  if $s _h^k\notin \cU_H^k$, we have 
 \begin{align*}
     c(s_h^k) - \tau \le \overline{c}^k(s_h^k) + 2\sqrt{\frac{2}{N^k(s_h^k)}\log(SK/\delta)}-\tau \le 2\sqrt{\frac{2}{N^k(s_h^k)}\log(SK/\delta)}.
 \end{align*}
 Thus we can get
 \begin{align*}
     N^k(s_h^k) \le \frac{8}{(c(s)-\tau)^2}\log(SK/\delta).
 \end{align*}
 Now denote $\mathcal{K}$ be the set of episodes $k$ that $s_h^{k+1} \in \Delta_h^k(s_h^k,a_h^k)$.
 Now for each $s \notin \cU$, denote $k_s = \max\{k: \exists \ h \in [H], s = s_h^k\}$, 
 then our total violation can be bounded by 
 \begin{align*}
     \sum_{k=1}^K \sum_{h=1}^H (c(s_h^k)-\tau)_+ &\le S^2AH^2 + \sum_{k \in \mathcal{K}} \sum_{h=1}^H (c(s_h^k)-\tau)_+\\
     &\le S^2AH^2 + \sum_{s \notin \cU}N^{k_s}(s)(c(s)-\tau)_+\\
     &\le S^2AH^2 + \sum_{s \notin \cU}\frac{8}{(c(s)-\tau)_+^2}(c(s_h^k)-\tau)_+\log(SK/\delta)\\
     &\le S^2AH^2 + \frac{8S}{d}\log(SK/\delta).
 \end{align*}
 
 \end{proof}

 \subsection{Proof of Theorem~\ref{thm:safe UCRL}}
 First, apply the same argument in the Appendix 
 \ref{sec:proof of safe RL}, the violation during the exploration phase can be bounded by 

     \begin{align*}
     \sum_{k=1}^K\sum_{h=1}^{H} (c(s_h^k)-\tau)_{+}&\le S^2AH^2+\sum_{k=1}^K \sum_{h=1}^H \sqrt{\frac{2}{N^k(s_h^k)}\log(SK/\delta)}\\
     &\le S^2AH^2+ \sqrt{2\log(SK/\delta)}\sum_{s \in \cS}\sum_{i=1}^{N^K(s)}\sqrt{\frac{1}{i}}\\
     &\le S^2AH^2+ \sqrt{8\log(SK/\delta)}\sum_{s \in \cS} \sqrt{N^K(s)}\\
     & = \widetilde{\cO}(S^2AH^2+\sqrt{ST}).
 \end{align*}

 Now we prove the Algorithm~\ref{alg:safe UCRL} will return $\varepsilon$-safe and $\varepsilon$-optimal policy. 
 Lemma~\ref{lemma:error less than W} shows that $\overline{W}_h^{k}(s_1,\pi_1^{k+1}(s_1))$ bounds the estimation error for any reward function $r$. 
 \begin{lemma}\label{lemma:error less than W}
    For any feasible policy $\pi \in \Pi^k$,  step $h$, episode $k$ and reward function $r$, defining  the estimation error as 
       $ \hat{e}^{k,\pi}_h(s,a;r) = |\hat{Q}_h^{k,\pi}(s,a;r)-Q_h^{k,\pi}(s,a;r)|$.
    Then, with probability at least $1-\delta$, we have 
    \begin{equation*}
        \hat{e}^{k,\pi}_h(s,a;r) \le \overline{W}_h^k(s,a).
    \end{equation*}
\end{lemma}

First, we prove the Lemma~\ref{lemma:error less than W}. 
We provide an lemma to show a high probability event for concentration.
\begin{lemma}[\cite{ACT2020adaptiveRFE}]
   Define the event as  
    \begin{align*}
        \cE_4 = \left\{\forall k, h, s, a, \mbox{KL}(\hat{\PP}_h^k(\cdot \mid (s,a)),\PP_h^k(\cdot \mid (s,a))\le \frac{\gamma(N_h^k(s,a),\delta)}{N_h^k(s,a)}\right\},
    \end{align*}
    where $\beta(n,\delta) = 2\log(SAHK/\delta)+(S-1)\log(e(1+n/(S-1)))$.
    Then $\Pr(\cE_4^c) \le \delta$.
\end{lemma}
By Bellman equations,
\begin{align*}
    \hat{Q}_h^{\pi}(s,a;r)-Q_h^{\pi}(s,a;r) = \sum_{s'}(\hat{\PP}_h(s'\mid s,a)-\PP_h(s'\mid s,a))Q_{h+1}^{\pi}(s',\pi(s');r)\nonumber\\
    \ \ \ + \sum_{s'}\hat{\PP}_h(s'\mid s,a)(\hat{Q}_{h+1}^{\pi}(s',\pi(s');r)-Q_{h+1}^{\pi}(s',\pi(s');r)).
\end{align*}
By the Pinkser's inequality and event $\cE_4$, 
\begin{align*}
    \hat{e}_h^{\pi}(s,a;r)\le &\sum_{s'}|\hat{\PP}_h(s'\mid s,a)-\PP_h(s'\mid s,a)|Q_{h+1}^{\pi}(s',\pi(s');r)\nonumber\\
    &\ \ \ + \sum_{s'}\hat{\PP}_h(s'\mid s,a)|\hat{Q}_{h+1}^{\pi}(s',\pi(s');r)-Q_{h+1}^{\pi}(s',\pi(s');r)|\\
    &\le H\Vert \hat{\PP}_h(s'\mid s,a)-\PP_h(s'\mid s,a)\Vert_1 + \sum_{s'}\hat{\PP}_h(s'\mid s,a)\hat{e}_{h+1}^{\pi}(s',\pi_{h+1}(s');r)\\&
    \le H\sqrt{\frac{\gamma(N_h(s,a),\delta)}{N_h(s,a)}}+\sum_{s'}\hat{\PP}_h(s'\mid s,a)\hat{e}_{h+1}^{\pi}(s,\pi_{h+1}(s');r).
\end{align*}
By induction, if  $\hat{e}_{h+1}^{\pi}(s,a;r) \le \overline{W}_{h+1}(s,a)$, for $s \notin \cU_h^k, a \in A_{h}^{k,safe}(s)$, we can have 
\begin{align}
    \hat{e}_h^{\pi}(s,a;r)&\le \min\left\{H, H\sqrt{\frac{\gamma(N_h(s,a),\delta)}{N_h(s,a)}}+\sum_{s'}\hat{\PP}_h(s'\mid s,a)\hat{e}_{h+1}^{\pi}(s,\pi_{h+1}(s');r)\right\}\nonumber\\&\le
    \min\left\{H, H\sqrt{\frac{\gamma(N_h(s,a),\delta)}{N_h(s,a)}}+\sum_{s'}\hat{\PP}_h(s'\mid s,a)\max_{b \in A_{h+1}^{k,safe}(s')}\hat{e}_{h+1}^{\pi}(s', b;r)\right\}\label{ineq:key}\\
    &\le \min\left\{H, H\sqrt{\frac{\gamma(N_h(s,a),\delta)}{N_h(s,a)}}+\sum_{s'}\hat{\PP}_h(s'\mid s,a)\max_{b \in A_{h+1}^{k,safe}(s')}\overline{W}_{h+1}(s', b)\right\}\nonumber\\
    &\le \overline{W}_h(s,a).\nonumber
\end{align}
The second inequality is because $\hat{\PP}_h(s'\mid s,a) = 0$ for all $s' \in \cU_{h+1}$, and then for $s' \notin \cU_{h+1}, \pi_{h+1}(s') \in A_{h+1}^{k,safe}(s')$. The third inequality holds by induction hypothesis, and the last inequality holds by the definition of $\overline{W}_h(s,a)$.

Also, for other state-action pair $(s,a)$, we can get 
\begin{align*}
    \hat{e}_h^{\pi}(s,a;r)&\le \min\left\{H, H\sqrt{\frac{\gamma(N_h(s,a),\delta)}{N_h(s,a)}}+\sum_{s'}\hat{\PP}_h(s'\mid s,a)\hat{e}_{h+1}^{\pi}(s,\pi_{h+1}(s');r)\right\}\\&\le
    \min\left\{H, H\sqrt{\frac{\gamma(N_h(s,a),\delta)}{N_h(s,a)}}+\sum_{s'}\hat{\PP}_h(s'\mid s,a)\max_{b \in A}\hat{e}_{h+1}^{\pi}(s', b;r)\right\}\\
    &\le \min\left\{H, H\sqrt{\frac{\gamma(N_h(s,a),\delta)}{N_h(s,a)}}+\sum_{s'}\hat{\PP}_h(s'\mid s,a)\max_{b \in A}\overline{W}_{h+1}(s', b)\right\}\\
    &\le \overline{W}_h(s,a).
\end{align*}
Hence the Lemma~\ref{lemma:error less than W} is true. Now we assume the algorithm terminates at episode $K+1$, for any feasible policy $\pi \in \Pi^K$, if we define the value function in model $\hat{M}$ as $\hat{V}$, we will get 
\begin{align*}
    |V_1^{\pi}(s_1)-\hat{V}_1^{\pi}(s_1)|  &= |Q_1^{\pi}(s_1,\pi_1(s_1))-\hat{Q}_1^{\pi}(s_1,\pi_1(s_1)) |\\&\le \overline{W}_h^K(s_1,\pi_1(s_1))\\&\le \overline{W}_h^K(s_1,\pi^{K+1}_1(s_1))\\&\le \varepsilon/2.
\end{align*}
where the second inequality is because $\pi^{K+1}$ is the greedy policy with respect to $\overline{W}_h^K$ and $\pi$ is in the feasible policy set $\Pi^K$ at episode $K$. 
Then for any reward function $r$, denote the $\hat{\pi^*} \in \Pi^K$ to be the optimal policy in estimated MDP $\hat{M}$. Recall that $\pi^* \in \Pi^K\subseteq \Pi^*$, we will have 
\begin{align*}
    V_1^{\pi^*}(s_1)-V_1^{\hat{\pi^*}}(s_1) & \le \hat{V}_1^{\pi^*}(s_1) - V_1^{\hat{\pi^*}}(s_1) + \varepsilon/2\\
    &\le \hat{V}_1^{\hat{\pi^*}}(s_1)-V_1^{\hat{\pi^*}}+\varepsilon/2\\
    &\le\varepsilon.
\end{align*}

Now we consider the expected violation. We will show that for all feasible policy $\pi \in \Pi^K$, the expected violation will be upper bounded by $2\overline{W}_1(s_1,\pi_1(s_1)).$

We first consider the event for feasible policy $\pi$
\begin{align*}
    \cP_\pi = \{\forall \ h \in [H-1], s_{h+1}\in \Delta_h^K(s_h,a_h)\mid s_h, a_h\sim \pi\}.
\end{align*}
Assume that the event holds, then because $\pi \in \Pi^K$ is the feasible policy with respect to $\Delta_h^K(s,a)$, we have $s_h\notin \cU_h$ for all $1\le h\le H+1$. Then \begin{align*}
(c(s_h)-\tau)_{+}\le \sqrt{\frac{2}{N^K(s)}\log(SK/\delta)}.
\end{align*}
and under the event $\cP_\pi$
\begin{align}\label{eq:underevent P}
    \EE_{\pi}\left[\sum_{h=1}^{H} (c(s_h)-\tau)_{+}\right]\le \EE_{\pi}\left[\sum_{h=1}^{H} \sqrt{\frac{2}{N^K(s_h)}\log\left(\frac{SK}{\delta}\right)}\right].
\end{align}
 Now define 
 
 \begin{align*}
     &\ \ \ \ \ \overline{F}_h^{\pi}(s,a) \\&= \min\left\{H, \sqrt{\frac{2}{N(s_h)}\log\left(\frac{SK}{\delta}\right)}+ H\sqrt{\frac{\gamma(N_h(s,a),\delta)}{N_h(s,a)}}+\sum_{s'}\hat{\PP}_h(s'\mid s,a) F_{h+1}^\pi(s',\pi_{h+1}(s'))\right\}.
 \end{align*}

 Since $\sqrt{\frac{2}{N(s_h)}\log\left(\frac{SK}{\delta}\right)}\le H\sqrt{\frac{\gamma(N_h(s,a),\delta)}{N_h(s,a)}}$, we have $\overline{F}_h^{\pi}(s,a) \le \overline{W}_h(s,a)$ by recursion. The proof is similar to \ref{lemma:error less than W}.

 Also denote $$F_h^{\pi}(s,a) = \EE_\pi\left[\sum_{h'=h}^H \sqrt{\frac{2}{N^K(s_h)}\log\left(\frac{SK}{\delta}\right)}\ \Bigg|\  s_h=s, a_h = a\right].$$ Next we prove $\overline{F}_h^\pi(s,a)\ge F_h^\pi(s,a).$
    To show this fact, we use the induction. When $h = H+1$, $F_{H+1}^\pi(s,a) = \overline{F}_{H+1}^\pi(s,a) = 0$. Now assume that $\overline{F}_{h+1}^\pi(s,a) \ge F_{h+1}^\pi(s,a)$, then 
    \begin{align*}
        &\ \ \ \ \ \overline{F}_h^\pi(s,a) \\&= \min\left\{H, \sqrt{\frac{2}{N(s)}\log\left(\frac{SK}{\delta}\right)}+ H\sqrt{\frac{\gamma(N_h(s,a),\delta)}{N_h(s,a)}}+\sum_{s'}\hat{\PP}_h(s'\mid s,a) \overline{F}_{h+1}^\pi(s',\pi_{h+1}(s'))\right\}\\
        &\ge \min\left\{H, \sqrt{\frac{2}{N(s)}\log\left(\frac{SK}{\delta}\right)}+ H\sqrt{\frac{\gamma(N_h(s,a),\delta)}{N_h(s,a)}}+\sum_{s'}\hat{\PP}_h(s'\mid s,a) F_{h+1}^\pi(s',\pi_{h+1}(s'))\right\}\\
        &\ge \min\left\{H, \sqrt{\frac{2}{N(s_h)}\log\left(\frac{SK}{\delta}\right)}+ H\Vert \hat{\PP}_h(s'\mid s,a) -\PP_h(s'\mid s,a)\Vert_1 \right.\\ &\hspace{20em}\left.+\sum_{s'}\hat{\PP}_h(s'\mid s,a) F_{h+1}^\pi(s',\pi_{h+1}(s'))\right\}\\
        &\ge \min\left\{H, \sqrt{\frac{2}{N(s_h)}\log\left(\frac{SK}{\delta}\right)}+ \sum_{s'}\PP_h(s'\mid s,a) F_{h+1}^\pi(s',\pi_{h+1}(s'))\right\}\\
        & = F_h^\pi(s,a).
    \end{align*}
    Thus by induction, we know $F_1^\pi(s_1,\pi_1(s_1))\le \overline{F}_1^\pi(s_1,\pi_1(s_1)) \le \overline{W}_1^K(s,\pi_1(s_1))$ because 
    $\sqrt{\frac{2}{N(s)}\log\left(\frac{SK}{\delta}\right)}\le H\sqrt{\frac{\gamma(N_h(s,a),\delta)}{N_h(s,a)}}$.

From now, we have proved that under the condition $\cP_\pi$, the expected violation can be upper bounded by $\overline{W}_1(s,\pi_1(s_1))$. Next we state that $\Pr(\cP_\pi^c)$ is small. Recall that $\cP_\pi = \{\forall \ h \in [H-1], s_{h+1}\in \Delta_h^K(s_h,a_h)\mid \pi\}$, we define 
\begin{align*}G_h^\pi(s,a) = \Pr\{\exists \ h' \in [h,H-1], s_{h'+1}\notin \Delta_h^K(s_{h'},a_{h'})\mid s_h=s, a_h = a\}
\end{align*}
and 
\begin{align*}
    &\ \ \ \ \ \overline{G}_h^\pi(s,a) \\&= \min\left\{H,\frac{SH\log(S^2HA/\delta)}{N_h(s,a)}+H\sqrt{\frac{\gamma(N_h(s,a),\delta)}{N_h(s,a)}}+ \sum_{s'}\hat{\PP}_h(s'\mid s,a) \overline{G}_{h+1}^\pi(s',\pi_{h+1}(s'))\right\}.
\end{align*}
Then by recursion, we can easily show that $\overline{G}_h^\pi(s,a)\le \overline{W}_h^K(s,a)$ by  $\log(S^2HA/\delta) \le \gamma(N_h(s,a),\delta)$.
\begin{lemma}\label{lemma:bound G h pi}
    For all $h \in [H]$ and feasible policy $\pi$, we have $\overline{G}_h^\pi(s,a) \ge HG_h^\pi(s,a)$.
\end{lemma}
\begin{proof}
    Define $\widetilde{G}_h^\pi(s,a)$ as 
    \begin{align*}
        \widetilde{G}_h^\pi(s,a) = \min\left\{H,\frac{SH\log(S^2HA/\delta)}{N_h(s,a)}+ \sum_{s'}\PP_h(s'\mid s,a) \widetilde{G}_{h+1}^\pi(s',\pi_{h+1}(s'))\right\}.
    \end{align*}
    We will show $\overline{G}_h^\pi(s,a)\ge \widetilde{G}_h^\pi(s,a)\ge HG_h^\pi(s,a).$
We prove these two inequalities by induction. First, for $h = H+1$, these inequalities holds. Now assume that they hold for $h+1$, then
\begin{align*}
    &\ \ \ \ \ \overline{G}_h^\pi(s,a)  \\&= \min\left\{H,\frac{SH\log(S^2HA/\delta)}{N_h(s,a)}+H\sqrt{\frac{\gamma(N_h(s,a),\delta)}{N_h(s,a)}}+ \sum_{s'}\hat{\PP}_h(s'\mid s,a) \overline{G}_{h+1}^\pi(s',\pi_{h+1}(s'))\right\}\\
    & \ge \min\left\{H,\frac{SH\log(S^2HA/\delta)}{N_h(s,a)}+H\sqrt{\frac{\gamma(N_h(s,a),\delta)}{N_h(s,a)}}+ \sum_{s'}\hat{\PP}_h(s'\mid s,a) \widetilde{G}_{h+1}^\pi(s',\pi_{h+1}(s'))\right\}\\
    &\ge \min\left\{H,\frac{SH\log(S^2HA/\delta)}{N_h(s,a)}+H\Vert \hat{\PP}_h(s'\mid s,a) -\PP_h(s'\mid s,a)\Vert_1\right.\\&\hspace{20em}\left.+ \sum_{s'}\hat{\PP}_h(s'\mid s,a) \widetilde{G}_{h+1}^\pi(s',\pi_{h+1}(s'))\right\}\\
    &\ge \min\left\{H,\frac{SH\log(S^2HA/\delta)}{N_h(s,a)}+ \sum_{s'}\PP_h(s'\mid s,a) \widetilde{G}_{h+1}^\pi(s',\pi_{h+1}(s'))\right\}\\
    & = \widetilde{G}_h^\pi(s,a).
\end{align*}
Now to prove that $\widetilde{G}_h^\pi(s,a)\ge HG_h^\pi(s,a)$, we need the following lemma.
\begin{lemma}\label{lemma:Delta(sa)}
    Fixed a step $h \in [H]$ and state-action pair $(s,a) \in \cS\times \cA$, then with high probability at least $1-\delta$, for any $s' \in \Delta_h(s,a)$ with $P(s'\mid s,a)\ge \frac{\log(S/\delta)}{N_h(s,a)}$, $s' \in \Delta^K_h(s,a)$.
\end{lemma}
\begin{proof}
    Since we sample $s'$ from $P_h(s'\mid s,a)$ for $N_h(s,a)$ times, if $s'\notin \Delta^K_h(s,a)$, then $s'$ are not sampled. The probability that $s'$ are not sampled will be $(1-p)^{N_h(s,a)}$, where $p =\PP_h(s'\mid s,a)$. When $p\ge \frac{\log(S/\delta)}{N_h(s,a)}$, the probability will be at most 
    \begin{equation*}
        (1-p)^{N_h(s,a)}\le \left(1-\frac{\log(S/\delta)}{N_h(s,a)}\right)^{N_h(s,a)}\le e^{-\log(S/\delta)} = \frac{\delta}{S}.
    \end{equation*}
    Taking the union bound for all possible next state $s'$, the lemma has been proved.
\end{proof}
By the Lemma~\ref{lemma:Delta(sa)}, define the event as 
\begin{align*}
    \cE_5 = \left\{\forall \ h,s,a \mbox{and}\ s' \in \Delta_h(s,a)\ \mbox{with}\  P(s'\mid s,a)\ge \frac{\log(S^2HA/\delta)}{N_h(s,a)}: s' \in \Delta_h^K(s,a)\right\}.
\end{align*}
Then $\Pr(\cE_5^c)\le \delta$. Under the event $\cE_5$, at step $h \in [H]$ and state-action pair $(s_h,a_h)$, we have 
\begin{align}\label{ineq:bound probability}
    \PP(s_{h+1}\notin \Delta_h^K(s_h,a_h))\le \frac{S\log(S^2HA/\delta)}{N_h(s_h,a_h)}.
\end{align}
This inequality is because $s_{h+1}\notin \Delta_h^K(s_h,a_h)$ only happens when $P(s'\mid s,a) \le \frac{\log(S^2HA/\delta)}{N_h(s_h,a_h)}$ under the event $\cE_5$.
Now we first decompose the $G_h^\pi(s,a)$ as
\begin{align*}
    H\cdot G_h^\pi(s,a) &= \Pr\{\exists \ h' \in [h,H-1], s_{h'+1}\notin \Delta_h^K(s_{h'},a_{h'})\mid s_h=s, a_h = a, \pi\}\\
    &\le H\sum_{h'=h}^{H-1} \Pr\{s_{h'+1}\notin \hat{Delta}_{h'}^K(s_{h'},a_{h'})\mid s_h = s, a_h = a, \pi\}\\
    &\le \EE_\pi \left[\sum_{h'=h}^{H-1}\frac{SH\log(S^2HA/\delta)}{N_{h'}(s_{h'},a_{h'})}\right].
\end{align*}
The first inequality is to decompose the probability into the summation over step $h$, and the second inequality is derived by Eq.~\eqref{ineq:bound probability}.
Now denote $g_h^\pi(s,a) = \min\left\{H,\EE_\pi \left[\sum_{h'=h}^{H-1}\frac{SH\log(S^2HA/\delta)}{N_{h'}(s_{h'},a_{h'})}\right]\right\}$, then we can easily have $g_h^\pi(s,a)\le \widetilde{G}_h^\pi(s,a)$ by recursion. In fact, if 
\begin{align*}
    &\ \ \ \ \ g_h^\pi(s,a) \\&= \min\left\{H,\EE_{\pi} \left[\sum_{h'=h}^{H-1}\frac{SH\log(S^2HA/\delta)}{N_{h'}(s_{h'},a_{h'})}\ \Bigg|\  (s_h,a_h) = (s,a)\right]\right\}\\
    &= \min\left\{H, \frac{SH\log(S^2HA/\delta)}{N_{h}(s,a)}+\EE_\pi\left[\sum_{h'=h+1}^{H-1}\frac{SH\log(S^2HA/\delta)}{N_{h'}(s_{h'},a_{h'})}\ \Bigg|\  (s_h,a_h) = (s,a)\right]\right\}\\
    & = \min\left\{H, \frac{SH\log(S^2HA/\delta)}{N_{h}(s,a)}\right.\\&\hspace{5em}\left.+\min\left\{H,\EE_\pi\left[\sum_{h'=h+1}^{H-1}\frac{SH\log(S^2HA/\delta)}{N_{h'}(s_{h'},a_{h'})}\ \Bigg|\  (s_h,a_h) = (s,a)\right]\right\}\right\}\\
    &=\min\left\{H, \frac{SH\log(S^2HA/\delta)}{N_{h}(s,a)}+\sum_{s' \in \cS}\PP_h(s'\mid s,a)g_{h+1}^\pi(s',\pi_{h+1}(s'))\right\}
\end{align*}
has the same recursive equation as $\widetilde{G}_h^\pi(s,a)$. Since $g_H^\pi(s,a) = 0\le \widetilde{G}_H^\pi(s,a)$, by induction we can get $g_h^\pi(s,a)\le \widetilde{G}_h^\pi(s,a)$, and then 
\begin{align*}
    H\cdot G_h^\pi(s,a) \le g_h^\pi(s,a)\le \widetilde{G}_h^\pi(s,a)
\end{align*}
holds directly. The Lemma~\ref{lemma:bound G h pi} has been proved.

Now we return to our main theorem. Consider the policy $\pi \in \Pi^K$ is one of feasible policy at episode $K$.
For each possible $\rT_H = \{s_1, a_1, \cdots, s_H, a_H\}$, denote $\Pr_{\pi}(\rT_H)$ be the probability for generating $\rT_H$ following policy $\pi$ and true transition kernel. 
If the event $s_{h+1}\notin \Delta_h^K(s_h,a_h)$ happens, we define the trajectory $\rT_H$ are ``unsafe". Denote the set of all ``unsafe" trajectories are $\cU$, we have 
\begin{align}
    &\EE_\pi\left[\sum_{h=1}^H (c(s_h)-\tau)_{+}\right] \nonumber\\ &\quad=\sum_{\rT_H}\Pr_{\pi}(\rT_H)\left[\sum_{h=1}^H (c(s_h)-\tau)_{+}\right]\nonumber\\
    &\quad\le\sum_{\rT_H \in \cU}\Pr_{\pi}(\rT_H)\left[\sum_{h=1}^H (c(s_h)-\tau)_{+}\right] + \sum_{\rT_H \notin \cU}\Pr_{\pi}(\rT_H)\left[\sum_{h=1}^H (c(s_h)-\tau)_{+}\right]
    \nonumber\\
    &\quad\le H\sum_{\rT_H \in \cU}\Pr_{\pi}(\rT_H) + \sum_{\rT_H \notin \cU}\Pr_{\pi}(\rT_H)\left[\sum_{h=1}^H \sqrt{\frac{2}{N^K(s_h)}\log(SK/\delta)}\right]
   \label{eq:P event}\\
    &\quad\le H\Pr_\pi\{\cP_\pi\} + \sum_{\rT_H}\Pr_{\pi}(\rT_H)\left[\sum_{h=1}^H \sqrt{\frac{2}{N^K(s_h)}\log(SK/\delta)}\right]\nonumber\\
    &\quad\le H\cdot G_1^\pi(s,a) + F_1^\pi(s,a)\label{eq:use F and G}\\
    &\quad\le 2\overline{W}_1^K(s_1,\pi_1(s_1))\nonumber\\
    &\quad\le 2\overline{W}_1^K(s_1,\pi^{K+1}_1(s_1))\nonumber\\
    &\quad\le \varepsilon.\nonumber
\end{align}
The Eq.~\eqref{eq:P event} is because $\rT_H \in \cU$ implies that the event $\cP_\pi$ happens, and we can apply Eq.~\eqref{eq:underevent P}. The Eq.~\eqref{eq:use F and G} is derived by the definition of $F_1^\pi(s,a)$ and $G_1^\pi(s,a)$. 
Until now, we have proved that if the algorithm terminates, all the requirements are satisfied. Now we turn to bound the sample complexity. We first assume that $K\ge SA$

Define the average $q_h^k = \sum_{(s,a)}\PP_h^{\pi^{k+1}}(s,a)\overline{W}_h^k(s,a)$.
Since the $\pi^{k+1}$ are greedy policy with respect to $\overline{W}_h^k(s,a)$, we have 
\begin{align*}
    &\overline{W}_h^k(s,a) \\&\quad= \min\left\{H, M(N_h^k(s,a), \delta) +  \sum_{s'}\hat{\PP}_h^k(s'\mid s,a) \overline{W}_{h+1}^k(s',\pi^{t+1}_{h+1}(s'))\right\}\\
    &\quad\le \min\left\{H, M(N_h^k(s,a), \delta) +  \sum_{s'}\PP_h^k(s'\mid s,a) \overline{W}_{h+1}^k(s',\pi^{t+1}_{h+1}(s'))\right.\\&\quad\hspace{20em}\left.+ H\Vert\hat{\PP}_h^k(s'\mid s,a)-\PP_h^k(s'\mid s,a)\Vert_1\right\}\\
    &\quad\le M'(N_h^k(s,a), \delta) +  \sum_{s'}\PP_h^k(s'\mid s,a) \overline{W}_{h+1}^k(s',\pi^{t+1}_{h+1}(s'))
\end{align*}
where $M'(n,\delta) =  3H\left(\sqrt{\frac{2\beta(n,\delta)}{n}}\wedge 1\right)+SH\left(\frac{\beta(n,\delta)}{n}\wedge 1\right)$.
Now 
\begin{align*}
    q_h^k &= \sum_{(s,a)}\PP_h^{\pi^{k+1}}(s,a)\overline{W}_h^k(s,a)\\
    & \le \sum_{(s,a)}\PP_h^{\pi^{k+1}}(s,a)M'(N_h^k(s,a),\delta) + \sum_{(s',s,a)}\PP_h^{\pi^{k+1}}(s,a)\PP_h^k(s'\mid s,a) \overline{W}_{h+1}^k(s',\pi_{h+1}^{k+1}(s'))\\
    & = \sum_{(s,a)}\PP_h^{\pi^{k+1}}(s,a)M'(N_h^k(s,a),\delta) + q_{h+1}^k.
\end{align*}
Thus 
\begin{align*}
    \varepsilon/2 <  q_1^k\le  \sum_{h=1}^H \sum_{(s,a)}\PP_h^{\pi^{k+1}}(s,a)M'(N_h^k(s,a),\delta), \forall \ k \in [K-1],
\end{align*}
where the first equality is because $q_1^k = \overline{W}_h^k(s_1,\pi^{k+1}(s_1))>  \varepsilon/2$. Sum over $k \in [K-1]$, we can get
\begin{align*}
    K\varepsilon/2\le \sum_{k=0}^{K-1}\sum_{h=1}^H \sum_{(s,a)}\PP_h^{\pi^{k+1}}(s,a)M'(N_h^k(s,a),\delta).
\end{align*}
Define the pseudo-count as $\overline{N}_h^k(s,a) = \sum_{i=1}^k\PP_h^{\pi^{i}}(s,a)$
and the event about the pseudo-count like \cite{ACT2020adaptiveRFE}:
\begin{align*}
    \cE_{cnt} = \left\{\forall \ k, h,s,a: N_h^k(s,a) \ge \frac{1}{2}\overline{N}_h^k(s,a)-\log(SAH/\delta) \right\}.
\end{align*}
Then $\Pr\{\cE_{cnt}\}\le \delta$. The proof is provided in \cite{ACT2020adaptiveRFE}. Also, the Lemma 7 in \cite{ACT2020adaptiveRFE} show that we can change $N_h^k(s,a)$ to $\overline{N}_h^k(s,a)$ up to a constant.
\begin{lemma}[\cite{ACT2020adaptiveRFE}]\label{lemma:N to barN}
    Under the event $\cE_{cnt}$, for any $s \in \cS, a \in \cA$ and $k$, if $\beta(x,\delta)/x$ is non-increasing for $x\ge 1$ and $\beta(x,\delta)$ is increasing, we have 
    \begin{align*}
        \frac{\gamma(N_h^k(s,a),\delta)}{N_h^k(s,a)}\wedge 1 \le  \frac{4\gamma(\overline{N}_h^k(s,a),\delta)}{\overline{N}_h^k(s,a)\vee 1}.
    \end{align*}
\end{lemma}
The proof of Lemma~\ref{lemma:N to barN} can be found in \cite{ACT2020adaptiveRFE}. 
Then we can have 
\begin{align}
    K\varepsilon/2&\le \sum_{k=0}^{K-1}\sum_{h=1}^H \sum_{(s,a)}\PP_h^{\pi^{k+1}}(s,a)M'(N_h^k(s,a),\delta)\\
    &\le \underbrace{3H\sum_{h=1}^H \sum_{k=0}^{K-1}\sum_{(s,a)} (\overline{N}_h^{k+1}(s,a) - \overline{N}_h^{k}(s,a))\sqrt{\frac{2\gamma(N_h^k(s,a),\delta)}{N_h^k(s,a)}\wedge 1}}_\text{(a)} \nonumber\\&\ \ \ \ \ \ \ \ \ \ + \underbrace{SH\sum_{h=1}^H \sum_{k=0}^{K-1}\sum_{(s,a)} (\overline{N}_h^{k+1}(s,a) - \overline{N}_h^{k}(s,a))\left(\frac{\gamma(N_h^k(s,a),\delta)}{N_h^k(s,a)}\wedge 1\right)}_\text{(b)}.
\end{align}
For part $(a)$, we upper bound it by
\begin{align}
    \mbox{(a)} &\le
    3H\sum_{h=1}^H \sum_{k=0}^{K-1}\sum_{(s,a)} (\overline{N}_h^{k+1}(s,a) - \overline{N}_h^{k}(s,a))\sqrt{\frac{8\gamma(\overline{N}_h^k(s,a),\delta)}{\overline{N}_h^k(s,a)\vee 1}}
    \nonumber\\ 
    &\le 3\sqrt{2}H\sqrt{\gamma(K,\delta)}\sum_{h=1}^H \sum_{k=0}^{K-1}\sum_{(s,a)}\frac{(\overline{N}_h^{k+1}(s,a) - \overline{N}_h^{k}(s,a))}{\sqrt{\overline{N}_h^k(s,a)\vee 1}}\label{eq:lemma gap }\\
    &\le 3(2+\sqrt{2})H\sqrt{\gamma(K,\delta)}\sum_{h=1}^H \sum_{(s,a)}\sqrt{\overline{N}_h^{K}(s,a)\vee 1}\nonumber\\
    &\le 3(2+\sqrt{2})H\sqrt{\gamma(K,\delta)}\sum_{h=1}^H \sum_{(s,a)}(1+\sqrt{\overline{N}_h^K(s,a)})\nonumber\\
    &\le 3(2+\sqrt{2})H^2\sqrt{\gamma(K,\delta)}(SA+\sqrt{SAK})\label{eq:cauchy technique}\\
    &\le 6(2+\sqrt{2})H^2\sqrt{\gamma(K,\delta)}\sqrt{SAK},\nonumber
\end{align}
where Eq.~\eqref{eq:lemma gap } is because Lemma 19 in \cite{Jaksch2010RL}, Eq.~\eqref{eq:cauchy technique} is derived by $\sum_{(s,a)}\overline{N}_h^K(s,a) = K$ and Cauchy's inequality,
and the last inequality is because we assume $K\ge SA$.

Also for part $(b)$, 
\begin{align}
    \mbox{(b)}&\le SH\sum_{h=1}^H \sum_{k=0}^{K-1}\sum_{(s,a)} (\overline{N}_h^{k+1}(s,a) - \overline{N}_h^{k}(s,a))\frac{4\gamma(\overline{N}_h^k(s,a),\delta)}{\overline{N}_h^k(s,a)\vee 1}\nonumber\\
    &\le SH\gamma(K,\delta)\sum_{h=1}^H \sum_{(s,a)}\sum_{k=0}^{K-1}\frac{(\overline{N}_h^{k+1}(s,a) - \overline{N}_h^{k}(s,a))}{\overline{N}_h^k(s,a) \vee 1}\nonumber\\
    &\le SH\gamma(K,\delta)\sum_{h=1}^H \sum_{(s,a)}4\log(\overline{N}_h^K(s,a)+1)\label{eq:lemma gap fast}\\
    &\le 4S^2AH^2\gamma(K,\delta)\log (K+1)\nonumber,
\end{align}
where the Eq.~\eqref{eq:lemma gap fast} is because the Lemma 9 in \cite{Pierre2020Fastactive}.
Now from the upper bound of $(a)$ and $(b)$, we can get
\begin{align}
    K\varepsilon/2 \le 6(2+\sqrt{2})H^2\sqrt{\gamma(K,\delta)}\sqrt{SAK}+4S^2AH^2\gamma(K,\delta)\log (K+1).\nonumber
\end{align}

Recall that $\gamma(n,\delta) = 2\log(SAHK/\delta)+(S-1)\log(e(1+n/(S-1)))\ge 1$, by some technical transformation, performing Lemma~\ref{lemma:transform technique} we can get
\begin{align*}
    K = \widetilde{\cO}\left(\left(\frac{H^4SA}{\varepsilon^2}+\frac{S^2AH^2}{\varepsilon}\right)\left(\log\left(\frac{1}{\delta}\right)+S\right)\right),
\end{align*}
under the event $\cE_{cnt}, \cE_1, \cE_4,\cE_5$. Hence the theorem holds with probability at least $1-4\delta$, and we can prove the theorem by replacing $\delta$ to $\delta/4$.

\end{proof}
\subsection{Proof of Theorem~\ref{thm:lower bound of violation}}
\begin{proof}


    We construct n MDP models with the same transition model but different safety cost. For each model, the states consist of a tree with degree $A$ and layer $H/3$. For each state inside the tree (not a leaf), the action $a_i$ will lead to the $i-$th branch of the next state. For leaf nodes, they are absorbing states, and we call them $s_1, s_2,\cdots, s_n$. Then $n = A^{H/3} = \frac{S(A-1)+1}{A}\ge \frac{S}{2}$.
\begin{figure}[H]
        \centering
        \includegraphics[scale = 0.8]{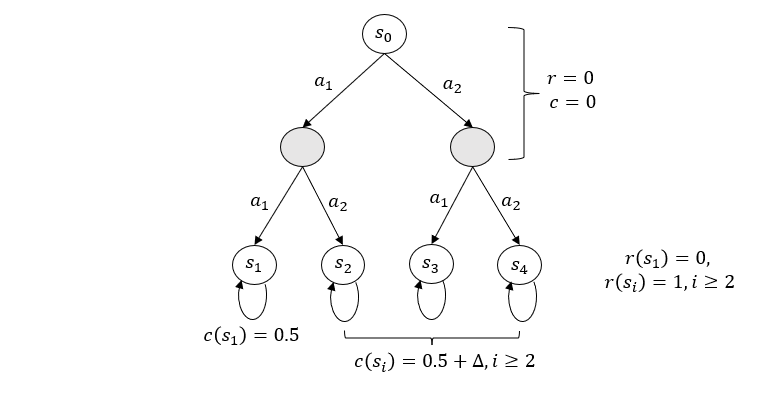}
        \caption{MDP instance $\tau_1$}
        \label{fig:lowerbound1 instance tau1}
    \end{figure}
    \begin{figure}[H]
        \centering
        \includegraphics[scale = 0.8]{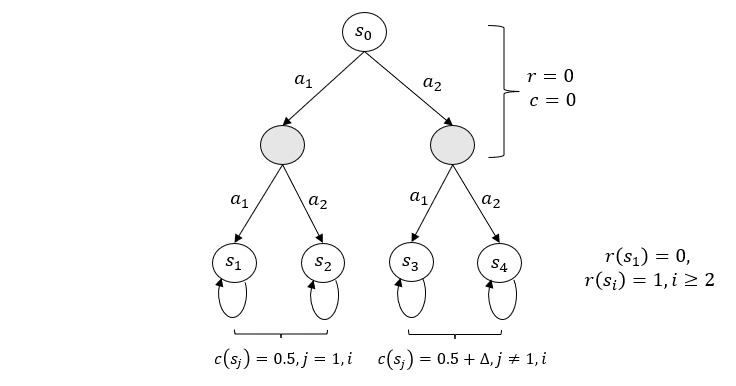}
        \caption{MDP instance $\tau_i (i=2)$}
        \label{fig:lowerbound1 instance taui}
    \end{figure}
    Now we first define the instance $\tau_1$ as follows: For each state inside the tree, the safety cost and reward are all 0. For leaf node, the cost function for state $s_i$ are $c(s_i) = 1/2+\Delta\mathbb{I}\{i\neq 1\}$. The reward function are $r(s_i) = \II\{i\neq 1\}$, which implies that only state $s_1$ has the reward 0 and other states $s_2,\cdots,s_n$ have reward 1.
    Now we define the instance $\tau_j(2\le j\le n)$. For state inside the tree, the safety cost and reward are equal to the instance $\tau_1$. For leaf node, the cost function for state $s_i$ are $c(s_i) = 1/2+\Delta\mathbb{I}\{i\neq 1, j\}$, $r(s_i)$ unchanged. Then for $\tau_j$, the state $s_1$ and $s_j$ (and initial state $s_0$) are safe. Then the safe optimal policy is to choose $a_j$ at the first step. Thus the safe optimal reward is $H$.
    The MDP instances are shown in Figure~\ref{fig:lowerbound1 instance tau1} and Figure~\ref{fig:lowerbound1 instance taui}.

    Denote $H'= \frac{2}{3}H$ are the rest of length when we arrive at the leaf states $s_1,\cdots,s_n$ .
    Let $R_K(\pi,\tau)$ represents the regret till step episode $K$ for instance $\tau$ and policy $\pi$, and $C_K(\pi,\tau)$ represents the violation. 
    For $\tau_1$, we can have $C_K(\pi, \tau_1)\ge \mathbb{P}_{\tau_1}(t_1(K)\le K/2)\cdot \frac{KH'\Delta}{2},$ where $t_i(K)$ means the times to choose the path to $s_i$. Also, $R_K(\pi, \tau_i) \ge \mathbb{P}_{\tau_i}(t_1(K)>K/2)\frac{KH'}{2}$.

    Applying Bertagnolle-Huber inequality, we have 
    \begin{align*}
        \mathbb{P}_{\tau_1}(t_1(K)\le K/2) + \mathbb{P}_{\tau_i}(t_1(K)>K/2) \ge \frac{1}{2}\exp\{-D( \mathbb{P}_{\tau_1}, \mathbb{P}_{\tau_i})\}.
    \end{align*}
    We choose $i$ such that $i = \arg\min_i \mathbb{E}_{\tau_1}[t_i(K)]$, then by $\sum_{i\neq 1}\mathbb{E}_{\tau_1}[t_i(K)]\le K$, $$\mathbb{E}_{\tau_1}[t_i(K)]\le \frac{K}{n-1}.$$
    By the definition of $ \mathbb{P}_{\tau_1}$and $ \mathbb{P}_{\tau_i}$, 
    \begin{align*}
        D( \mathbb{P}_{\tau_1}, \mathbb{P}_{\tau_i}) &=H' \mathbb{E}_{\tau_1}[t_i(K)]D(\mathcal{N}(1/2+\Delta,1), \mathcal{N}(1/2, 1)) \\&\le H'\mathbb{E}_{\tau_1}[t_i(K)]\cdot 4\Delta^2\\
        &\le \frac{4KH'\Delta^2}{n-1}.
    \end{align*}
    Thus choose $\Delta = \sqrt{n-1/4KH'}\le \frac{1}{2}$, (need $n\le T$) 
    \begin{align*}
        \mathbb{P}_{\tau_1}(t_1(K)\le K/2) + \mathbb{P}_{\tau_i}(t_1(K)>K/2) \ge \frac{1}{2}\exp\{-D( \mathbb{P}_{\tau_1}, \mathbb{P}_{\tau_i})\}\ge \frac{1}{2e}.
    \end{align*}
    If $\mathbb{E}_{\pi}(R_K(\pi, \tau))\le \frac{T}{24} = \frac{KH}{24} \le \frac{KH'}{16}$ for all $\tau$, then 
    \begin{align*}
        \mathbb{P}_{\tau_i}(t_1(K)>K/2)\le \frac{2R_K(\pi, \tau_i)}{KH'}\le \frac{1}{8}.
    \end{align*}
    Thus $$\mathbb{P}_{\tau_1}(t_1(K)\le K/2)\ge \frac{1}{2e}-\frac{1}{8}:= C$$ and
    \begin{align*}
        C_K(\pi, \tau_1)\ge \frac{CKH'\Delta}{2} = \frac{C\sqrt{(n-1)KH'}}{4} = \Omega(\sqrt{SHK}) = \Omega(\sqrt{SHK}),
    \end{align*}
    where $n \ge \frac{S}{2}$.
\end{proof}

\subsection{Proof of Theorem~\ref{thm:lower bound regret}}
\begin{proof}
    Construct the MDP $\tau_1, \tau_2,\cdots,\tau_n$ as follows: For each instance, the states consist of a tree with degree $A-1$ and layer $H/3$. For each state inside the tree (not a leaf), the action $a_i, 1\le i\le A-1$ will lead to the $i$-th branch of the next state. For the leaf nodes, denote them as $s_1,s_2,\cdots,s_n$, where $n = (A-1)^{H/3} = \frac{(S-3)(A-2)+1}{A-1} = \Omega(S)$ and these states will arrive at the final absorbing state $s_A$ and $s_B$ with reward $r(s_A)=0$ and $r(s_B)=1$ respectively.
    Also, the action $a_A$ will always lead the agent to the unique unsafe state $s_U$ with $r(s_U) = 0.5+\Delta'$ and $c(s_U)=1$ for some parameter $\Delta'$ in the states except two absorbing states.
    For all states except the unsafe state, the safety cost is 0.
    
    Now for the instance $\tau_1$, the transition probability between the leaf states and absorbing states are 
    \begin{align*}
        \left\{\begin{array}{c}
            p(s_A\mid s_1,a) = 0.5-\Delta  \\
            p(s_B\mid s_1,a) = 0.5+\Delta\\
            p(s_A\mid s_i,a) = p(s_B\mid s_1,a) = 0.5 \ \ \ \ i\ge 2
        \end{array}\right.
    \end{align*}
     For instance $\tau_j$, the transition probability are 
    \begin{align*}
        \left\{\begin{array}{c}
            p(s_A\mid s_1,a) = 0.5-\Delta  \\
            p(s_B\mid s_1,a) = 0.5+\Delta\\
            p(s_A\mid s_j,a) = 0.5-2\Delta  \\
            p(s_B\mid s_j,a) = 0.5+2\Delta\\
            p(s_A\mid s_i,a) = p(s_B\mid s_1,a) = 0.5 \ \ \ \ i\neq 1,j
        \end{array}\right.
    \end{align*}
    \begin{figure}[t]
        \centering
        \includegraphics[scale=0.8]{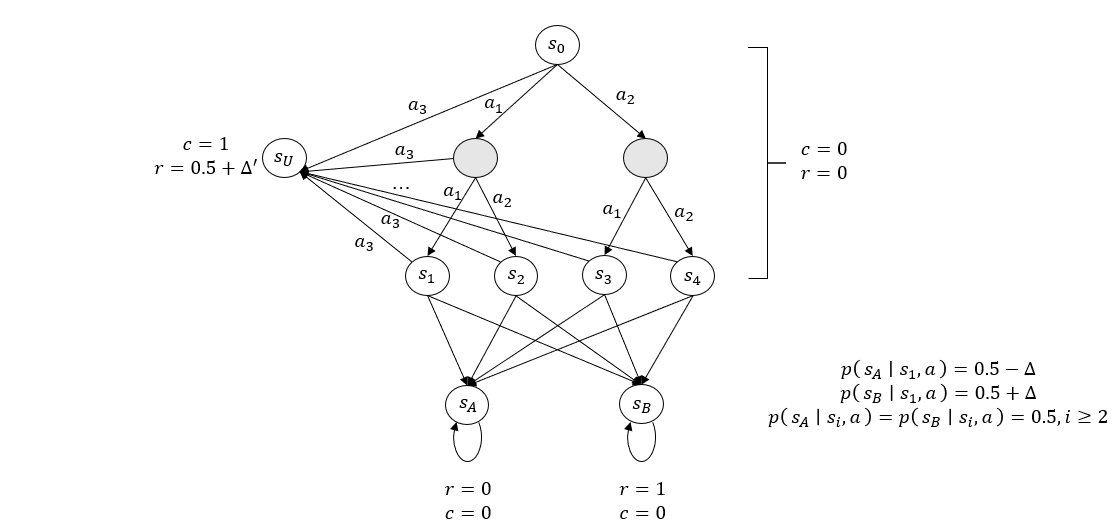}
        \caption{MDP instance $\tau_1$}
        \label{fig:lowerbound2 tau1}
    \end{figure}
    \begin{figure}[t]
        \centering
        \includegraphics[scale=0.8]{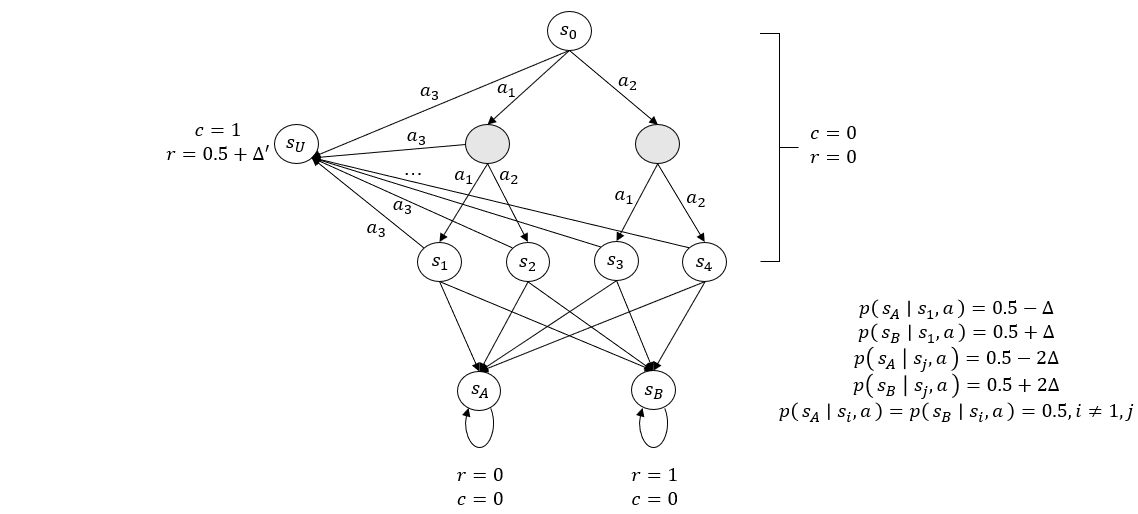}
        \caption{MDP instance $\tau_i$}
        \label{fig:lowerbound2 taui}
    \end{figure}

Then if we do not consider the constraint, choosing $a_A$ and getting into the unsafe state $s_U$ is always the optimal action, 
then define $\mbox{gap}_{h,s,a} = V_h^*(s) - Q_h^*(s,a)$, then the minimal gap \cite{nips2019gapdependent} is $$\mbox{gap}_{\min} = \min_{h,s,a}\{\mbox{gap}_{h,s,a}>0:\mbox{gap}_{h,s,a}\}    
                   \ge H'((0.5+\Delta')-(0.5+2\Delta))\ge H'(\Delta'-2\Delta).$$ Thus the regret will be at most $O(\frac{S^2A\cdot poly(H)\log T}{\Delta'-\Delta})$ by classical gap-dependent upper bound \cite{nips2019gapdependent}. As we will show later, we will choose $\Delta = \Theta(T^{-1/2})$ and $\Delta' = \Theta(T^{\frac{\alpha-1}{2}})$, hence the regret will be 
                   \begin{align*}
                    O\left(\frac{S^2A\cdot poly(H)\log T}{\Delta'-\Delta}\right) = \widetilde{\cO}\left(T^{\frac{1-\alpha}{2}}\right).
                   \end{align*}

Now if we consider the safe threshold $\tau = 0$, the violation will be the number of time to arrive the unsafe state. Now consider any algorithm, if it achieves $O(T^{1-\alpha})$ violation, they will arrive at unsafe state at most $O(T^{1-\alpha})$ times. Since denote $T_i(K), 1\le i\le n+1$ as the number of times when agent arrive at state $s_i$ till episode $K$. Then $\mathbb{E}_{\pi}[T_{n+1}(K)] \le CT^{1-\alpha}$. Note that each episode we will arrive at least one of $s_i, 1\le i\le n+1$ once, we know $\sum_{i=1}^{n+1}T_i(K) \ge K$, and 
$\sum_{i=1}^n T_i(K) \ge K-C\cdot T^{1-\alpha}\ge K/2$ when $K\ge (2C)^{1/\alpha}H^{\frac{1-\alpha}{\alpha}}$.

Now similar to proof of Theorem~\ref{thm:lower bound of violation}, let $C_K(\pi,\tau)$ represents the violation with algorithm $\pi = \{\pi_1,\cdots,\pi_K\}$. Denote $H'=\frac{2}{3}H$. For $\tau_1$, we have 
\begin{align*}R_K(\pi, \tau_1)&\ge \mathbb{P}_{\tau_1}(T_1(K)\le K/4) \frac{KH'\Delta}{4}-(\Delta'-\Delta)\mathbb{E}_{\pi}[T_{n+1}(K)]\\
&\ge \mathbb{P}_{\tau_1}(T_1(K)\le K/4) \frac{KH'\Delta}{4} - (\Delta'-\Delta)CT^{1-\alpha}\end{align*}

\begin{align*}R_K(\pi, \tau_i)&\ge \mathbb{P}_{\tau_i}(T_1(K)> K/4) \frac{KH'\Delta}{4}-(\Delta'-2\Delta)\mathbb{E}_{\pi}[T_{n+1}(K)]\\
&\ge \mathbb{P}_{\tau_i}(T_1(K)> K/4) \frac{KH'\Delta}{4} - (\Delta'-2\Delta)CT^{1-\alpha}\end{align*}

Also, we have 
\begin{align*}
    \mathbb{P}_{\tau_1}(T_1(K)\le K/4)+\mathbb{P}_{\tau_i}(T_1(K)> K/4)\ge \frac{1}{2}\exp\{-D(\mathbb{P}_{\tau_1}, \mathbb{P}_{\tau_i})\}
\end{align*}
Now WLOG we have $\mathbb{E}_{\tau_1}[T_i(K)]\le \frac{K}{2(n-1)}$
\begin{align*}
    D(\mathbb{P}_{\tau_1},\mathbb{P}_{\tau_i}) \le \mathbb{E}_{\tau_1}[T_i(K)]2\Delta^2\le \frac{K}{2(n-1)}\cdot 2\Delta^2 = \frac{K\Delta^2}{n-1}
\end{align*}
Choose $\Delta = \sqrt{\frac{n-1}{K}}$, then 
\begin{align*}
    \mathbb{P}_{\tau_1}(T_1(K)\le K/4)+\mathbb{P}_{\tau_i}(T_1(K)> K/4)\ge \frac{1}{2e}.
\end{align*}
Then at least one of $\tau_1$ and $\tau_i$ (assume $\tau_1$) have 
\begin{align*}
    R_K(\pi, \tau_1)\ge \frac{H'\sqrt{K(n-1)}}{8e} - C(\Delta'-2\Delta)T^{1-\alpha}
\end{align*}

Choose $\Delta' = T^{\frac{\alpha-1}{2}}$.  Then assume $T\ge 2^{\frac{2}{1-\alpha}}$ so that $c\le \frac{1}{2}.$
\begin{align*}
    R_K(\pi, \tau_1)\ge \frac{H'\sqrt{K(n-1)}}{8e} - CT^{(1-\alpha)/2} = \Omega(H\sqrt{SK}) = \Omega(\sqrt{HST}).
\end{align*}
\end{proof}
\section{Safe Markov Games}\label{appendix:safe markov games}
\subsection{Definitions}
In this section, we provide an extension for our general framework of Safe-RL-SW. We solve the safe zero-sum two-player Markov games~\cite{yu2021provably}. 

In zero-sum two-player Markov games, there is an agent and  another adversary who wants to make the agent into the unsafe region and also minimize the reward.  The goal of agent is to maximize the reward and avoid the unsafe region with the existence of adversary. To be more specific, define the action 
 space for the agent and adversary are $\cA$ and $\cB$ respectively. In each episode the agent starts at an initial state $s_1$. At each step $h \in [H]$, the agent and adversary observe the state $s_h$ and taking action $a_h \in \cA$ and $b_h \in \cB$ from the strategy $\mu_h(\cdot \mid s)$ and $\upsilon_h(\cdot \mid s)$ simultaneously. Then the agent receive a reward $r_h(s,a,b)$ and the environment transitions to the next state $s_{h+1}\sim P(s_{h+1}\mid s,a,b)$ by the transition kernel $P$. 

 One motivation of this setting is that, when a robot is in a dangerous environment, the robot should learn a completely safe policy even with the disturbance from the environment.  
Like the single-player MDP, we define the value function and Q-value function:
\begin{align*}
    V_h^{\mu, \upsilon}(s) &= \EE_{\pi}\left[\sum_{h'=h}^H r_{h'}(s_{h'}, a_{h'}, b_{h'})\Bigg| s_h = s, \mu, \upsilon\right].\\
    Q_h^{\mu, \upsilon}(s,a,b) &= \EE_{\pi}\left[\sum_{h'=h}^H r_{h'}(s_{h'}, a_{h'}, b_{h'})\Bigg| s_h = s, a_h = a, b_h= b, \mu, \upsilon\right].
\end{align*}
The cost function is a bounded function defined on the states $c(s) \in [0,1], s \in \cS$, and the unsafe states are defined $\cU_H =\{s \mid c(s)>\tau\}$ for some fixed threshold $\tau$.
Similar to the Section~\ref{sec:model and analysis of safe RL}, we define the possible next states for state $s$, step $h$, action pair $(a,b)$ as 
\begin{align*}
    \Delta_h(s,a,b) = \{s'\mid\PP_h(s'\mid s,a,b)>0\}.
\end{align*}

Hence we can define the unsafe states for step $h$ recursively as:
\begin{align*}
    \cU_h = \cU_{h+1}\cup \{s\mid \ \forall \in \cA, \exists\ b \in \cB, \Delta_{h}(s,a,b)\cap \cU_{h+1} \neq \emptyset\}.
\end{align*}
Then to keep safety, the agent cannot arrive at the state $s \in \cU_h$ at the step $1,2,\cdots,h$.
The safe action for the agent $A_h^{safe}(s)$ should avoid the unsafe region no matter what action the adversary takes.
\begin{align*}
    A_h^{safe}(s) = \{a \in \cA\mid \ \forall \ b \in \cB,  \Delta_h(s,a,b)\cap \cU_{h+1} = \emptyset\}.
\end{align*}
Hence a feasible policy should keep the agent in safe state $s \notin \cU_h$ at step $h$, and always chooses $a \in A_h^{safe}(s)$ to avoid the unsafe region regardless of the adversary. 
Analogous to the safe RL in Section~\ref{sec:model and analysis of safe RL}, assume the the optimal Bellman equation can be written as 
\begin{align*}
    Q^*_h(s,a,b) &= r(s,a) + \PP_h(s'\mid s,a,b)V^*_{h+1}(s').
    \\
    V^*_h(s) &= \max_{a \in A_h^{safe}(s)}\min_{b \in \cB} Q^*_h(s,a,b).
\end{align*}
If we denote $r(s,a,b) = -\infty$ for unsafe states $s \in \cU_H$, 
fixed an strategy $\upsilon$ for the adversary, the {\em best response} of the agent can be defined as $\arg\max_{\mu} V_1^{\mu, \upsilon}(s_1)$. Similarly, a {\em best response} of the adversary given agent's strategy $\mu$ can be defined as $\arg\min_{\upsilon}V_1^{\mu, \upsilon}(s_1)$. By the minimax equality, for any step $h \in [H]$ and $s \in \cS$, 
\begin{align*}
    \max_\mu \min_\upsilon V_h^{\mu, \upsilon}(s) =  \min_\upsilon \max_\mu V_h^{\mu, \upsilon}(s).
\end{align*}
A policy pair $(\mu^*, \upsilon^*)$ satisfy that $V_h^{\mu^*, \upsilon^*}(s) = \max_\mu \min_\upsilon V_h^{\mu, \upsilon}(s)$ for all step $h \in [H]$ is called a {\em Nash Equilibrium}, and the value $V_h^{\mu^*, \upsilon^*}(s)$ is called the {\em minimax value}. It is known that minimax value is unique for Markov games, and the agent's goal is make the reward closer to the minimax value.

Assume the agent and adversary executes the policy $(\mu_1,\upsilon_1),  (\mu_2,\upsilon_2),\cdots, (\mu_K, \upsilon_K)$ for episode $1,2,\cdots,K$, the regret is defined as 
\begin{align*}
    R(K) = \sum_{k=1}^K V_1^{\mu^*, \upsilon^*}(s_1) - V_1^{\mu_k, \upsilon_k}(s_1).
\end{align*}
Also, the actual state-wise violation is defined as 
\begin{align*}
    C(K) = \sum_{k=1}^K\sum_{h=1}^H  (c(s_h^k)-\tau)_{+}.
\end{align*}
\subsection{Algorithms and Results}
The main algorithm is presented in Algorithm~\ref{alg:markov}. The key idea is similar to the Algorithm~\ref{alg:safe RL}. However, in each episode we need to calculate an optimal policy $\pi^k$ by a sub-procedure ``Planning". This sub-procedure calculates the Nash Equilibrium for the estimated MDP. Note that there is always a mixed  Nash equilibrium strategy $(\mu, \upsilon)$, and the ZERO-SUM-NASH function in ``Planning" procedure calculates the Nash equilibrium for a Markov games. 

In the sub-procedure ``Planning", we calculate the $(Q_h^k)'(s,a,\cdot) = -\infty$ if there exists at least one adversary action $b$ such that $\Delta(s,a,b)\cap \cU_{h+1}\neq \emptyset$. In fact, the agent cannot take these actions at this time, hence by setting $(Q_h^k)'(s,a,\cdot)$ for these actions $a$, the support of Nash Equilibrium policy $\mu_h^k$ will not contains $a$. Otherwise, the minimax value is $-\infty$, and it means that $a$ is now in $\cU_h^k$, then at this time the agent can chooes the action arbitrarily
\begin{algorithm}[t]
\begin{algorithmic}[1]
\caption{Safe Markov Games}
	\label{alg:markov}
 \STATE Input: $\Delta_h^1(s,a,b) = \emptyset, N_h^{1}(s,a,b), N_h^1(s,a,b,s') = 0.$
 
 \FOR{$k=1,2,\cdots,K$}
	    \STATE Receive the initial state $s_1$. 
	    
	    \STATE Update the estimation $\hat{c}(s)$ based on history and get optimistic estimation $\bar{c}(s) = \hat{c}(s) - \beta(N^k(s), \delta)$ for each state $s$, where $\beta$ is the bouns function.

     \STATE Define $\cU_{H}^k = \{s\mid \bar{c}(s) > \tau\}$.

     \STATE Based on $\Delta_h^K(s,a)$, calculate $\cU_h^k = \cU_{h+1}^k\cup \{s\mid \forall a \in A, \Delta_k(s,a,b)\cap \cU_{h+1}^k\neq\emptyset\}$.
     \STATE $\hat{\PP}_h(s'\mid s,a,b) = \frac{N_h^k(s,a,b,s')}{N_h^k(s,a,b)}$.
     \STATE $\pi_h^k = $Planning$(k, r, \{\cU^k_h\}_{h \in [H]}, \Delta_h^K(s,a,b), N^k(s,a,b), \hat{\PP})$.
     
     \FOR{$h = 1,2,\cdots,H$}
        \STATE Take action $a_h^k$ such that $(a_h^k, b)\sim \pi_h^k(\cdot,\cdot\mid s)$ and arrive state $s_{h+1}^k\sim P(s\mid s_h^k, a_h^k, b_h^k)$, where $b_h^k$ is the action of the adversary.
        \STATE $\Delta_h^{k+1}(s_h^k,a_h^k, b_h^k) = \Delta_h^k(s_h^k,a_h^k,b_h^k) \cup \{s_{h+1}^k\}$.
     \ENDFOR
     \STATE Update $N_h^{k+1}(s,a,b)$, $ N_h^{k+1}(s,a,b,s'), \Delta_h^{k+1}(s,a,b)$ for all $(s,a,b)$.
     \ENDFOR
\end{algorithmic}
\end{algorithm}

\begin{algorithm}[t]
\begin{algorithmic}[1]
	\caption{Planning$(k, r, \{\cU_h\}_{h \in [H]}, \Delta(s,a,b), N, \hat{\PP})$}
	\label{alg:planning}
	\STATE Input $k, r, \{\cU_h\}, \Delta(s,a,b), N, \hat{\mathbb{P}}$.
 
	\FOR{$h = H,H-1,\cdots,1$}

            \FOR{$(s,a,b) \in \cS\times \cA\times \cB$}
                \STATE $\beta(N_h^k(s,a,b),\delta) \leftarrow 7H\sqrt{\log(5SABT/\delta)/N_h^k(s,a,b)}$.

                \STATE $Q_h^k(s,a,b) = \min\{(r_h + \sum_{s'}\hat{\PP}_h^k(s'\mid s,a,b)V_{h+1}(s'))+\beta(N_h^k(s,a,b),\delta), H\}$.
            \ENDFOR

        \FOR{$s \in S$}
            \FOR{$a \in A$}
            \IF{$\exists b, \Delta_h(s,a,b) \cap \cU_{h+1}\neq \emptyset$}\STATE $(Q_h^k)'(s,a,\cdot) = -\infty$.
            \ELSE
            
            \STATE $(Q_h^k)'(s,a,\cdot) = Q_h^k(s,a,\cdot)$, $a \in A_h^k(s)$.
            \ENDIF
            
        \ENDFOR
        \STATE $\mu_h^k, \upsilon_h^k = \mbox{ZERO-SUM-NASH}((Q_h^k)'(s,\cdot,\cdot)).$
        
        \STATE $V_h^k(s) = \EE_{a\sim \mu_h^k(\cdot\mid s), b\sim \upsilon_h^k(\cdot \mid s)}(a,b)Q_h(s,a,b)$.
        \ENDFOR
         \ENDFOR
        \end{algorithmic}
\end{algorithm}

Our main result are presented in the following theorem, which shows that our algorithm achieves both $O(\sqrt{T})$ regret and $O(\sqrt{ST})$ state-wise violation.
\begin{theorem}
With probability at least $1-\delta$, the Algorithm~\ref{alg:markov} have regret and state-wise violation 
\begin{align*}
    R(T) &= O(H^3\sqrt{SABT})\\
    C(T) &= O(\sqrt{ST}+S^2ABH^2),
\end{align*}
where $A=|\cA|$, $B = |\cB|$ and $S = |\cS|.$
\end{theorem}
\subsection{Proofs}
\begin{proof}
    We first prove that, with probability at least $1-\delta$, $Q_h^k(s,a,b)\ge Q_h^*(s,a,b)$ and $V_h^k(s)\ge V_h^*(s)$ for any episode $k \in [K].$
    We prove this fact by induction. Denote $\DD_{\mu \times \upsilon}Q(s) = \EE_{a\sim \mu(\cdot \mid s), b \sim \upsilon(\cdot \mid s)}Q(s,a,b)$. The statement holds for $h = H+1$. Suppose the bounds hold for $Q-$values in the step $h+1$, now we consider step $h$.

    \begin{align*}
        V_{h+1}^k(s) &= \mathbb{D}_{\pi_h^k}Q_{h+1}^k(s)\\
        & = \max_{supp(\mu) \subseteq A_h^{k,safe}(s)}\mathbb{D}_{\mu \times v_h^k} Q_{h+1}^k(s)\\
        & \ge \max_{\mu \in supp(\mu)\subseteq A_h^k(s)}\mathbb{D}_{\mu \times v_h^k} Q_{h+1}^*(s)\\
        &\ge  \max_{\mu \in supp(\mu)\subseteq A_h^*(s)}\mathbb{D}_{\mu \times v_h^k} Q_{h+1}^*(s)\\
        &\ge \min_v\max_{\mu \in supp(\mu)\subseteq A_h^*(s)}\mathbb{D}_{\mu\times v}Q_{h+1}^*(s)\\
        & = V_{h+1}^*(s),
    \end{align*}
    where $A_h^{k,safe}(s)$ is all safe action for max-player at episode $k$ and step $h$ it estimated, $A_h^*(s)$ is the true safe action, and $supp(\mu) = \{a \in \cA\mid \mu(a)>0\}$ From our algorithm, we know $A_h(s)\subseteq A_h^k(s)$ for all time if the confidence event always holds.

    Then 
    \begin{align*}
        Q_h^k(s,a,b) - Q_h^*(s,a,b) &= (\hat{\mathbb{P}}_h^k V_{h+1}^k - \mathbb{P}_h V_{h+1}^* + \beta_h^k)(s,a,b)\\
        &\ge (\hat{\mathbb{P}}_h^k - \mathbb{P}_h)(V_{h+1}^*)(s,a,b) + \beta(N_h^k(s,a,b),\delta)\\
        &\ge 0.
    \end{align*}
    The last inequality is because of the Chernoff-Hoeffding's inequality. 
    Thus define event $$\cE_6 = \left\{\ \forall k,h,s,a,b, \ Q_h^k(s,a,b)\ge Q_h^*(s,a,b), V_h^k(s)\ge V_h^*(s)\right\}.$$
    Then $\Pr\{\cE_6^c\}\le \delta.$

    Now we start to bound the regret. 
    \begin{align*}
        R(K) &= \sum_{k=1}^K V_1^{\mu^*, \upsilon^*}(s_1) - V_1^{\mu_k, \upsilon_k}(s_1)\\
        &\le \sum_{k=1}^K V_1^k(s_1) - V_1^{\mu_k, \upsilon_k}(s_1).
    \end{align*}
    Now we first assume $\upsilon_k$ is the best response of $\mu_k$. Otherwise our regret can be larger. We calculate the regret by \begin{align}&\ \ \ \ \ V_h^k(s_h^k)-V_h^{\mu_k, \upsilon_k}(s_h^k) \nonumber\\&= \DD_{\pi^k}(Q_h^k)(s_h^k) - \DD_{\mu_k\times \upsilon_k}Q_h^{\mu_k\times\upsilon_k}(s_h^k)\nonumber\\
    &\le \DD_{\mu_k\times \upsilon_k}(Q_h^k)(s_h^k) - \DD_{\mu_k\times \upsilon_k}Q_h^{\mu_k\times\upsilon_k}(s_h^k)\nonumber\\
    & = \DD_{\mu_k\times \upsilon_k}(Q_h^k-Q_h^{\mu_k\times \upsilon_k})(s_h^k)\nonumber\\
    &= (Q_h^k-Q_h^{\mu_k\times \upsilon_k})(s_h^k, a_h^k, b_h^k) + \alpha_h^k\nonumber\\
    &= (\hat{\PP}_h-\PP_h)V_{h+1}^k(s_h^k,a_h^k,b_h^k) + \PP_h\left(V_{h+1}^k-V_{h+1}^{\mu_k\times\upsilon_k}\right)(s_h^k, a_h^k, b_h^k) + \beta(N_h^k(s_h^k,a_h^k,b_h^k),\delta) \nonumber+\alpha_h^k\\
    &\le \left(V_{h+1}^{k}-V_{h+1}^{\mu_k\times\upsilon_k}\right)(s_{h+1}^k) + \alpha_h^k + \gamma_h^k + 2\beta(N_h^k(s_h^k,a_h^k,b_h^k),\delta) \nonumber\\&\hspace{20em}+ (\hat{\PP}_h-\PP_h)(V_{h+1}^k - V_{h+1}^*)(s_h^k,a_h^k,b_h^k),\nonumber
    \end{align}
    where $$\alpha_h^k = \PP_h\left(V_{h+1}^k-V_{h+1}^{\mu_k\times\upsilon_k}\right)(s_h^k, a_h^k, b_h^k) - \left(V_{h+1}^{k}-V_{h+1}^{\mu_k\times\upsilon_k}\right)(s_{h+1}^k),$$ $$\gamma_h^k = \left(\DD_{\mu_k\times \upsilon_k}(Q_h^k-Q_h^{\mu_k\times \upsilon_k})(s_h^k)-(Q_h^k-Q_h^{\mu_k\times \upsilon_k})(s_h^k)\right).$$
    The last inequality is because $(\hat{\PP_h}-\PP_h)V_{h+1}^*(s_h^k, a_h^k,b_h^k)\le \beta(N_h^k(s_h^k,a_h^k,b_h^k),\delta)$ by Chernoff-Hoeffding's inequality.

    Now by the similar analysis in the previous section, with probability at least $1-\delta$, the event 
    
    \begin{align*}
        \cE_7 &= \left\{\ \forall s, a, b,s', k, h, |\hat{\PP}_h(s'\mid s,a,b) -\PP_h(s'\mid s,a,b)|\right.\\&\hspace{5em}\left.\le \sqrt{\frac{2\PP_h(s'\mid s,a,b)}{N_h^k(s,a)}\log \left(\frac{2S^2ABHK}{\delta}\right)}+ \frac{2}{3N_h^k(s,a)}\log \left(\frac{2S^2ABHK}{\delta}\right)\right\}.
    \end{align*}
    
    holds with probability at least $1-\delta$ by Bernstein's inequality. Then for $\iota = \log \left(\frac{2S^2ABHK}{\delta}\right)$
    \begin{align*}
        &(\hat{\PP}_h-\PP_h)(V_{h+1}^k - V_{h+1}^*)(s_h^k,a_h^k,b_h^k) \\&\quad\le \sum_{s'}\left(\sqrt{\frac{2\PP_h(s'\mid s,a,b)}{N_h^k(s,a)}\iota}+ \frac{2}{3N_h^k(s,a)}\iota\right)(V_{h+1}^k - V_{h+1}^*)(s')\\
        & \quad\le \sum_{s'}\left(\frac{\PP_h(s'\mid s_h^k,a_h^k)}{H}+\frac{H\iota}{2N_h^k(s_h^k,a_h^k)}+\frac{2\iota}{3N_h^k(s_h^k,a_h^k)}\right)(V_{h+1}^k - V_{h+1}^*)(s')\\
        &\quad\le \frac{1}{H}\PP_h(V_{h+1}^k-V_{h+1}^{*})(s_h^k,a_h^k) + \frac{2H^2S\iota}{N_h^k(s_h^k,a_h^k)}\\
        &\quad\le \frac{1}{H}\PP_h(V_{h+1}^k-V_{h+1}^{\mu_k\times\upsilon_k})(s_h^k,a_h^k) + \frac{2H^2S\iota}{N_h^k(s_h^k,a_h^k)}.
    \end{align*}
    The last inequality is because we assume $\upsilon_k$ is the best response of $\mu_k$ and then $V_{h+1}^*(s_h^k,a_h^k)\ge V_{h+1}^{\mu_k\times\upsilon_k}(s_h^k,a_h^k).$
    Now we can get 
    \begin{align*}
        &V_h^k(s_h^k)-V_h^{\mu_k\times \upsilon_k}\\&\quad\le \left(1+\frac{1}{H}\right)(V_{h+1}^k(s_h^k)-V_{h+1}^{\mu_k\times \upsilon_k}) + 2\alpha_h^k + 2\gamma_h^k + 2\beta(N_h^k(s_h^k,a_h^k,b_h^k),\delta)+\frac{2H^2S\iota}{N_h^k(s_h^k,a_h^k)}
    \end{align*}
    and then 
    \begin{align*}
        V_1^k(s_1) - V_1^{\mu_k\times \upsilon_k}&\le \sum_{h=1}^H \left(1+\frac{1}{H}\right)^H (2\alpha_h^k + 2\gamma_h^k + 2\beta(N_h^k(s_h^k,a_h^k,b_h^k),\delta))+\frac{2H^2S\iota}{N_h^k(s_h^k,a_h^k)}\\
        &\le \sum_{h=1}^H e\cdot (2\alpha_h^k + 2\gamma_h^k + 2\beta(N_h^k(s_h^k,a_h^k,b_h^k),\delta))+\frac{2H^2S\iota}{N_h^k(s_h^k,a_h^k)}.
    \end{align*}
    By Azuma-hoeffding's inequality, with probability at least $1-2\delta$, $\sum_{k=1}^K \sum_{h=1}^H \alpha_h^k \le O(\sqrt{H^3K})$, $\sum_{k=1}^K \sum_{h=1}^H \gamma_h^k = O(\sqrt{H^3K})$ and $$\sum_{k=1}^K \sum_{h=1}^H \frac{2H^2S\iota}{N_h^k(s_h^k,a_h^k)} = O(\log K).$$
    Note that 
    \begin{align*}\sum_{k=1}^K \sum_{h=1}^H\beta(N_h^k(s_h^k,a_h^k,b_h^k),\delta)& = \sum_{k=1}^K \sum_{h=1}^H 7H\sqrt{\frac{\iota}{N_h^k(s_h^k,a_h^k,b_h^k)}}\\
    &\le O(H\iota \sum_{(h,s,a,b) \in [H]\times \cS\times \cA\times \cB}\sqrt{N_h^k(s_h^k,a_h^k,b_h^k)}\\
    &\le \widetilde{\cO}(\sqrt{H^3SABT}).\end{align*}
    For the violation, the argument is similar to the Theorem~\ref{thm:safe RL}. During the learning process, the $s_{h+1}^k\notin \Delta_h^K(s_h^k,a_h^k,b_h^k)$ can appear at most $S^2ABH$ times because each time the summation
    \begin{equation*}
        \sum_{(h,s,a,b)\in [H]\times \cS\times \cA\times \cB }|\Delta_h^K(s,a,b)|
    \end{equation*}
    will increase at least 1, and it has a upper bound $S^2ABH$.
    Thus it will lead to at most $S^2ABH^2$ regret. For other situations, the total violation can be upper bounded by $\widetilde{\cO}(\sqrt{ST})$. Thus the final violation bound is $\Tilde{O}(S^2ABH^2+\sqrt{ST}).$
    
\end{proof}

\section{Technical Lemmas}
\subsection{Lemma~\ref{lemma:transform technique}}
\begin{lemma}\label{lemma:transform technique}
    If \begin{align*}
        K\varepsilon/2 \le a\sqrt{K}\sqrt{\gamma(K,\delta)} + b\gamma(K,\delta)\log(K+1),
    \end{align*}
    where $\gamma(K,\delta) = 2\log(SAHK/\delta)+(S-1)\log(e(1+K/(S-1)))$, then 
    \begin{align*}
        K = \widetilde{\cO}\left(\left(\frac{a^2}{\varepsilon^2}+\frac{b}{\varepsilon}\right)\log\left(\frac{1}{\delta}\right)+\frac{Sb}{\varepsilon}\right),
    \end{align*}
    where $\widetilde{\cO}(\cdot)$ ignores all the term $\log S, \log A, \log H, \log 1/\varepsilon$ and $\log \log(1/\delta).$
\end{lemma}
\begin{proof}
If $K\le 4$, the lemma is trivial. Now assume $K\ge 4$, then $e(1+K)\le 4K$ 
    and \begin{align*}
        \gamma(K,\delta) &\le 2\log(SAHK/\delta)+S\log(e(1+K))\\&\le 2\log(SAHK/\delta)+S\log(4K) \\&\le 2S\log(4K)+\log(SAH/\delta).
    \end{align*}
Then
\begin{align*}
        K\varepsilon/2 &\le a\sqrt{K}\sqrt{\gamma(K,\delta)} + b\gamma(K,\delta)\log(K+1)\\
        &\le 2\max\{a\sqrt{K}\sqrt{\gamma(K,\delta)}, b\gamma(K,\delta)\log(K+1)\}.
    \end{align*}
\paragraph{Case 1:}If $K\varepsilon/2\le 2a\sqrt{K}\sqrt{\gamma(K,\delta)}$, we can get
\begin{align*}
    K&\le \frac{16a^2}{\varepsilon^2}(2S\log(4K)+\log(SAH/\delta))\\
    &\le 2\max\left\{\frac{32Sa^2}{\varepsilon^2}\log(4K), \frac{16a^2}{\varepsilon^2}\log(SAH/\delta)\right\}.
\end{align*}
\paragraph{Subcase 1:}If $K\le \frac{64Sa^2}{\varepsilon^2}\log(4K)$
by Lemma~\ref{lemma:small technique}, we can get $K = \widetilde{\cO}(\frac{Sa^2}{\varepsilon^2})$.
\paragraph{Subcase 2:}If $K\le \frac{16a^2}{\varepsilon^2}\log(SAH/\delta)$, we complete the proof.
\paragraph{Case 2:}If $K\varepsilon/2\le 2b\gamma(K,\delta)\log(K+1)$,
 we can get
 \begin{align*}
     K &\le \frac{4b}{\varepsilon}\log(4K)(2S\log(4K) + \log(SAH/\delta))\\
     &\le 2\max\left\{\frac{4b}{\varepsilon}\log(4K)2S\log(4K), \frac{4b}{\varepsilon}\log(4K)\log(SAH/\delta)\right\}.
 \end{align*}
 \paragraph{Subcase 3:}
 If $K\le \frac{8bS}{\varepsilon}\log^2(4K)$
By Lemma~\ref{lemma:small technique}, we can get $K = \widetilde{\cO}(\frac{8bS}{\varepsilon})$.

\paragraph{Subcase 4:} If $K\le \frac{8b}{\varepsilon}\log(4K)\log(SAH/\delta)$, then by Lemma~\ref{lemma:small technique}, we can get $K = \widetilde{\cO}(\frac{b}{\varepsilon}\log(1/\delta))$

From the four subcases, we complete the proof of Lemma~\ref{lemma:transform technique}.

\end{proof}
\subsection{Lemma~\ref{lemma:small technique}}
\begin{lemma}\label{lemma:small technique}
    If $K\le Q\log^2(4K)$ with $Q\ge 4$, then we have $K = \widetilde{\cO}(Q)$.
\end{lemma}
\begin{proof}
    Define $f(x) = \frac{x}{\log^2(4x)}$, then $f'(x) = \frac{\log^2(4x)-\log(4x)}{\log^4(4x)}>0$ when $4x\ge 3$, and $f(x)$ is increasing over $x\ge 1$. Then we only need to prove there exists a constant $c_2>0$ such that  
    \begin{align*}
        f(c_2Q\log^2 Q)\ge Q. 
    \end{align*}
    In fact, if this inequality holds, any $K$ such that $K\le Q\log^2(4K)$ should have $f(K)\le Q\le f(c_2Q\log^2Q)$, and then $K\le c_2Q\log^2Q = \widetilde{\cO}(Q)$.
    To prove this inequality, observe that 
    \begin{align*}
        Q\log^2(4c_2Q\log^2 Q) &= Q(\log 4c_2Q + \log\log^2 Q)^2 \\&\le 2Q\log^2(4c_2Q) + 2Q(\log\log^2 Q)^2\\
        &\le 4Q\log^2 (4c_2) + 4Q\log^2(Q)+2Q\log^2 Q\\
        &\le (4\log^2(4c_2)+6)Q\log^2Q.
    \end{align*}
    where the third inequality is because $\log^2Q \le Q$ for $Q\ge 4$.
    So if we choose $c_2$ to satisfy $c_2\ge 4\log^2(4c_2)+6$, we have $Q\log^2(4c_2Q\log^2 Q)\le c_2Q\log^2 Q$, which implies $f(c_2Q\log^2 Q)\ge Q$.
\end{proof}

\section{Detailed Analysis of Assumption~\ref{assum:existence of feasible pol}}\label{appendix:detailed analysis of assum}
In this section, we provide a more detailed analysis of the equivalence between $s_1\notin \cU_1$ and the existence of feasible policy. 
Indeed, if the agent gets into some state $s \in \cU_h$ at step $h$, then any action can lead her to an unsafe next state $s' \in \cU_{h+1}$. Recursively, any action sequence $a_h,a_{h+1},\cdots,a_{H-1}$ can lead the agent to unsafe states $s_{h+1} \in \cU_{h+1},\cdots,s_{H} \in \cU_{H}$. As a result, the agent cannot avoid getting into an unsafe state. Therefore, all feasible policies $\pi$ must satisfy that, at any step $h$, the probability for the agent in the state $s \in \cU_h$ should be zero under policy $\pi$. 
\begin{align}
    \PP_h^{\pi}(s) = 0, \ \ \forall \ s \in \cU_h.
\end{align}
Recall that \begin{align*}
    A_h^{safe}(s) = \{a \in \cA\mid \Delta_h(s,a)\cap \cU_{h+1}=\emptyset\}.
\end{align*}
From these definitions, a feasible policy $\pi$ should satisfy that $\pi_h(s) \in A_h^{safe}(s)$ for any safe state $s \notin \cU_h$. Moreover, for any safe state $s_h \notin \cU_h$, there exists at least one safe action $a_h \in A_h^{safe}(s_h)$, and all possible next states $s_{h+1} \in \Delta(s_h,a_h)$ satisfy $s_{h+1}\notin \cU_{h+1}$. Recursively, there exists at least one feasible action sequence $a_h \in A_h^{safe}(s_h), \cdots, a_H \in A_H^{safe}(s_H)$ which satisfies that $s_{h'} \notin \cU_{h'}$ for $h+1\le h'\le H$. Hence the assumption for the existence of feasible policy can be simplified to $s_1\notin \cU_1$.

\section{Experiment Setup}\label{appendix:experiment}

We perform each algorithm  for $20000$ episodes in a grid environment (\cite{chow2018nipslycontrol,wei2022triple}) with $25$ states and $4$ actions (four directions), and  report the average reward and violations across runs. 
In the experiment, we note that the algorithm Optpess-bouns~\citep{exploration2020CMDP} needs to solve a linear programming with size $O(SAH)$ for each episode, which is extremely slow. Thus we only compare all algorithms in a relatively small environment.
In the Safe-RL-SW experiments, we choose $\delta = 0.005$ and binary $0$-$1$ cost function. We set the safety threshold as 0.5 for both the CMDP algorithms and the SUCBVI algorithm.  

For the Safe-RFE-SW experiments, we compare our algorithm SRF-UCRL with a state-of-the-art RFE algorithm RF-UCRL~\citep{ACT2020adaptiveRFE}  in an MDP with 11 states and 5 actions. We choose $\delta = 0.005$ and safety threshold of SRF-UCRL as 0.5. We run 500 episodes 100 times, and then report the average reward in each episode and cumulative step-wise violation. Also at each episode, we calculate the expected violation for the outputted policy and plot the curves. 
All experiments are conducted  with AMD Ryzen 7 5800H with Radeon Graphics and a 16GB unified memory.

}
\end{document}